\documentclass[11pt]{article}
\pdfoutput=1
\usepackage{macro}

\usepackage{geometry}

\usepackage[english]{babel}
\usepackage[T1]{fontenc}
\usepackage[utf8]{inputenc}
\usepackage{authblk}

\usepackage{natbib}
\usepackage{url}
\usepackage{breqn}
\usepackage{centernot}
\usepackage{lscape}
\usepackage{longtable}
\usepackage{booktabs}
\usepackage{multirow}
\usepackage{lipsum}
\usepackage{algorithm, algorithmicx}
\usepackage{algpseudocode}

\usepackage{color}
\usepackage{enumerate}
\usepackage{graphicx, subfigure, caption}
\usepackage{subfigure}
\usepackage{amsfonts, amsmath, amsthm, bm, amssymb}
\usepackage{caption}
\usepackage{adjustbox}
\usepackage{threeparttable}
\usepackage{rotating}

\newcommand{\beginsupplement}{
	\newpage
	\setcounter{page}{1}
	\setcounter{section}{0}
	\renewcommand{\thesection}{S\arabic{section}}%
	\setcounter{theorem}{0}
	\setcounter{example}{0}
	\setcounter{property}{0}
	\setcounter{lemma}{0}
	\setcounter{proposition}{0}
	\setcounter{subtext}{0}
	\setcounter{remark}{0}
	\setcounter{corollary}{0}
	\setcounter{definition}{0}
	\setcounter{conjecture}{0}
	\setcounter{assumption}{0}
	\renewcommand{\thetheorem}{S\arabic{theorem}}
	\renewcommand{\thesubtext}{S\arabic{subtext}}
	\renewcommand{\theexample}{S\arabic{example}}
	\renewcommand{\theproperty}{S\arabic{property}}
	\renewcommand{\thelemma}{S\arabic{lemma}}
	\renewcommand{\theproposition}{S\arabic{proposition}}
	\renewcommand{\theremark}{S\arabic{remark}}
	\renewcommand{\thecorollary}{S\arabic{corollary}}
	\renewcommand{\thedefinition}{S\arabic{definition}}
	\renewcommand{\theconjecture}{S\arabic{conjecture}}
	\renewcommand{\theassumption}{S\arabic{assumption}}
	\setcounter{algorithm}{0}
	\setcounter{figure}{0}
	\setcounter{table}{0}
	\setcounter{equation}{0}
	\renewcommand{\thealgorithm}{S\arabic{algorithm}}
	\renewcommand{\thefigure}{S\arabic{figure}}
	\renewcommand{\thetable}{S\arabic{table}}
	\renewcommand{\theequation}{S\arabic{equation}}
	\onecolumn
	\pagestyle{empty}
}

\title{Robust Multiple Kernel $k$-means Clustering using Min-Max Optimization}
\author[1]{Seojin Bang\thanks{seojinb@cs.cmu.edu}}
\author[2]{Yaoliang Yu}
\author[1]{Wei Wu\thanks{weiwu2@cs.cmu.edu}}
\affil[1]{School of Computer Science, Carnegie Mellon University}
\affil[2]{School of Computer Science \& Waterloo AI Institute, University of Waterloo}

\begin{document}
\maketitle
\begin{abstract}
Multiple kernel learning is a type of multiview learning that combines different data modalities by capturing view-specific patterns using kernels. Although supervised multiple kernel learning has been extensively studied, until recently, only a few unsupervised approaches have been proposed. In the meanwhile, adversarial learning has recently received much attention. Many works have been proposed to defend against adversarial examples. However, little is known about the effect of adversarial perturbation in the context of multiview learning, and even less in the unsupervised case. In this study, we show that adversarial features added to a view can make the existing approaches with the $\min_{\Hv}$-$\min_{\thetav}$ formulation in multiple kernel clustering yield unfavorable clusters. To address this problem and inspired by recent works in adversarial learning, we propose a multiple kernel clustering method with the $\min_{\Hv}$-$\max_{\thetav}$ framework that aims to be robust to such adversarial perturbation. We evaluate the robustness of our method on simulation data under different types of adversarial perturbations and show that it outperforms several compared existing methods. In the real data analysis, We demonstrate the utility of our method on a real-world problem. 
\end{abstract}

\section{Introduction}
	In recent years, multiview (or multimodal) learning approaches have been developed to integrate abundant yet diverse data modality. Integrating diverse modalities is challenging because data from different sources (called \textit{views}) have different statistical properties. To address this problem, multiple kernel learning uses view-specific kernels to capture diverse patterns of multiple views \citep{lanckriet2004learning}. Then, it integrates views as a linear sum of multiple kernels weighted by kernel coefficients $\thetav$, and applies a standard classification or clustering algorithm to the combined kernel. Driven by advantages of using kernels, it has witnessed successes in various domains such
	as computer vision \citep{gehler2009feature} and document classification \citep{lanckriet2004learning}.
	
	While supervised multiple kernel learning has been extensively studied, only a few unsupervised approaches have been proposed until recently, among which, multiple kernel $k$-means clustering is one of the commonly used approaches. For simplicity, we limit our discussion here to the case of multiple kernel $k$-means clustering. Although details vary, they find clusters by alternately optimizing the kernel coefficients $\thetav$ and clustering assignment $\Hv$ as shown in Figure~\ref{fig:overview}. Existing works employ a $\min_{\Hv}$-$\min_{\thetav}$ (or $\max_{\Hv}$-$\max_{\thetav}$) framework. In particular, they first find a combination of views that reveals low within-cluster variance, and then find clusters minimizing such variance \citep{gonen2014localized,liu2017multiple,liu2016multiple,yu2012optimized,yao2018multiple}.
	
	Meanwhile, adversarial learning has received much attention in recent years. Plenty of studies have demonstrated that very small changes to input can make a model, in particular a deep learning model, to produce incorrect predictions \citep{biggio2013evasion,szegedy2014intriguing,goodfellow2015explaining}. This phenomenon is so-called \textit{adversarial example phenomenon}. 
	Many studies have proposed defence mechanisms resistant to adversarial example \citep{madry2018towards,sinha2018certifying,zhang2019theoretically}, in which they aim to minimize a loss under the maximum adversary. In particular, they use min-max framework that first finds adversarial examples that maximize a loss and then finds the model parameters that minimize the adversarial loss. In the context of deep learning, this min-max framework has become an effective approach to learn a robust model against adversarial attacks.
	
	Despite all these works in adversarial learning, little is known about adversary and robustness in the context of multiview learning, and even less in the unsupervised case. Inspired by recent works in adversarial learning, in this study, we show that adversarial features, e.g., a number of random noise or redundant variables, added to a certain view can deceive the existing $\min_{\Hv}$-$\min_{\thetav}$ methods. In particular, they make $\min_{\Hv}$-$\min_{\thetav}$ methods to ignore the view and find clusters that are largely determined by other views. For simplicity, we denote such features as \textit{adversarial perturbation}. 
	
	To address this problem, we propose a multiple kernel clustering method, \textit{multiple kernel $k$-means clustering with $\min_{\Hv}$-$\max_{\thetav}$ and $l_2$ regularization} (MML-MKKC). It aims to be robust to adversarial perturbation by using the $\min_{\Hv}$-$\max_{\thetav}$ formulation. In particular, the inner maximization finds a combination of multiple views in favor of a view (or views) that reveals high within-cluster variance, whereas the outer minimization finds clusters that minimize such variance. By capturing such variance while adversary is present, our method can mitigate the effect of adversarial perturbation (see details in Section~\ref{sec:robust}).
	
	We evaluate our method on the simulated multiview data with adversarial perturbations that allow us to assess robustness of our method. The result shows that our method outperforms the compared existing multiple kernel clustering methods and yields clusters by making good use of all views, including the view with the added perturbation. We also demonstrate the utility of our method on a real-world problem to identify cancer subtypes. 
	~\\~\\
	Our main contributions are as follow.
	\begin{itemize}
		\item To the best of our knowledge, this is the first work that studies adversarial perturbation in a unsupervised multiview setting. In particular, we examine the effect of potential adversaries on existing multiview clustering models.
		\item We found out that adversarial perturbation can make existing multiview clustering methods with the $\min_{\Hv}$-$\min_{\thetav}$ framework yield unfavorable results. They tend to ignore the view with adversarial perturbation and find clusters by relying largely on other views. 
		\item We propose a multiple kernel $k$-means clustering method MML-MKKC using a $\min_{\Hv}$-$\max_{\thetav}$ framework that aims to be robust to adversarial perturbations. This is achieved by minimizing the within-cluster variance in a combination of the views that reveals high within-cluster variance.
	\end{itemize}
	
\section{Method}\label{sec:method}
	In this section, we begin with introducing prior works. We then propose a multiple kernel clustering method that aims to be robust against adversarial perturbation.
	
	\subsection{Kernel $k$-means clustering}
	Let $\xv \in \Rc^p$ be a sample instance and $\phiv:\Rc^p \rightarrow \Fc$ is a nonlinear mapping of $\xv$ onto a reproducing kernel Hilbert space $\Fc$. By mapping to a higher dimensional feature space using $\phiv$, kernel $k$-means clustering linearly separates samples that were only non-linearly separable in the input space \citep{girolami2002mercer}. The optimization problem of kernel $k$-means clustering is the same as $k$-means clustering but replacing $\xv$ with a nonlinear mapping $\phiv(\xv) \in \Fc$, which is:
	\begin{align}
	\mini_{\Zv \in \cbb{0,1}^{n\times k}} \sum_{c=1}^{k}\sum_{i=1}^{n} z_{ic}|| \phiv \bb{\xv_i} - \muv_c ||^2_2 \quad
	\text{~s.t.~} \sum_{c=1}^{k}z_{ic} = 1 \nonumber
	\end{align}
	where $\xv_i$ is $i$-th sample instance, $z_{ic}$ is a binary cluster assignment for $i$-th sample and cluster $c$; $\muv_c=\sum_{i=1}^{n} z_{ic} \phiv \bb{\xv_i}/n_c$ is cluster center; $n_c=\sum_{i=1}^{n}z_{ic}$ is the size of cluster $c$; and $n$ is the number of samples. This is viewed as to minimize within-cluster variance in the feature space. This problem can be reformulated as a trace minimization \citep{zha2002spectral}: 
	\begin{align*}
	\mini_{\Zv \in \cbb{0,1}^{n\times k}} \textbf{tr} \bb{\Kv - \Lv^{1/2}\Zv^\top\Kv\Zv\Lv^{1/2}} \quad
	\text{~s.t.~} \Zv\mathbf{1}_k = \mathbf{1}_n
	\end{align*}
	where $\Zv = [z_{ic}]_{n\times k}$, $\Lv = \text{diag}\sbb{1/n_1, \cdots , 1/n_k}$, and $\Kv = \sbb{\phiv(\xv_i) \cdot \phiv(\xv_j)}_{n\times n}$. Unfortunately, this problem is NP-hard \citep{michael1979computers}. Note that $\Hv = \Zv\Lv^{1/2}$ represents normalized clustering assignment. Hence, we solve it by eliminating the discrete constraint on $\Hv$ while keeping the orthogonal constraint on $\Hv$:
	\begin{align}\label{exp:opt03}
	\mini_{\Hv \in \Rc^{n\times k}} \textbf{tr} \bb{\Kv - \Hv^\top\Kv\Hv}\quad
	\text{~s.t.~} \Hv^\top\Hv = \Iv_k
	\end{align}
	This is solved by a well-known result from \citet{fan1949theorem} (see Theorem~\ref{theo:kyfan}). The optimal solution is given by $\Hv = \Uv_k\Qv$ where each column of $\Uv_k = [\uv_1, \cdots, \uv_k]$ is eigenvectors of $\Kv$ involved with $k$ largest eigenvalues $\lambda_1 \geq \cdots \geq \lambda_k$ and $Q$ is an arbitrary orthogonal matrix. That is, the $k$ eigenvalues are one of the continuous solutions to the discrete cluster assignment \citep{ding2004k}. 
	After obtaining the continuous solution, the hard clustering assignment $\Zv$ is recovered by QR decomposition on $\Hv$ \citep{zha2002spectral} or by $k$-means clustering on normalized $\Hv$ \citep{ng2002spectral}.
	
	\subsection{Existing multiple kernel $k$-means clustering}\label{sec:general}
	Multiple kernel $k$-means clustering extends kernel $k$-means clustering, which has an additional procedure to combine multiple views. It captures view-specific similarity with different kernels and combines multiple kernels weighted by kernel coefficient $\thetav$. For example, it uses $\Kv_{\thetav} = \sum_{v=1}^{m}\theta^{(v)}\Kv^{(v)}$ or $\Kv_{\thetav} = \sum_{v=1}^{m}{\theta^{(v)}}^2\Kv^{(v)}$ where $\thetav^{(v)}$ is a (non-negative) kernel coefficient for view $v$. For a given $\Kv_{\thetav}$, it finds clusters that minimize within-cluster variance in the combined space. The problem is defined as follow:
	\begin{align}\label{exp:opt21}
	\mini_{\Hv \in \Rc^{n\times k}}&~\mini_{\thetav} \textbf{tr} \bb{\Kv_{\thetav} - \Hv^\top\Kv_{\thetav}\Hv}\\
	&\text{~s.t.~ } \Hv^\top\Hv = \Iv_k,~\thetav \geq \mathbf{0} \nonumber, f\bb{\thetav} \leq \mathbf{0} \nonumber
	\end{align}
	where $\thetav = \sbb{\theta^{(1)}, \cdots, \theta^{(m)}}^\top \in \Rc_+^m$, and $f\bb{\thetav} \leq \mathbf{0}$ is an \textit{appropriate} constraint on $\thetav$; without such constraint the inner minimization will have a trivial solution $\thetav = \mathbf{0}$. This problem is solved by alternately optimizing kernel coefficients $\thetav$ and clustering assignment matrix $\Hv$ given each other. 
	
	Existing methods are similar in that they all use the $\min_{\Hv}$-$\min_{\thetav}$ (or $\max_{\Hv}$-$\max_{\thetav}$) framework. \citet{gonen2014localized} captured the sample-specific characteristics by using sample-specific kernel coefficients. \citet{liu2017multiple} extended Gonen's approach to perform clustering under incomplete kernel matrices. \citet{liu2016multiple} used a matrix-induced $l_2$ regularization on $\thetav$ to avoid redundancy and improve the diversity of multiple kernels. \citet{yao2018multiple} incorporated a representative kernel selection process into multiple kernel $k$-means clustering to reduce redundancy and enhance the diversity of kernels. \citet{yu2012optimized} aimed to maximize between-cluster variance, hence, they used $\max_{\Hv}$-$\max_{\thetav}$, instead of $\min_{\Hv}$-$\min_{\thetav}$.
	
	\subsection{Robust multiple kernel $k$-means clustering}\label{sec:robust}
	We propose a multiple kernel $k$-means clustering method, MML-MKKC, that aims to be robust against adversarial perturbation. In order to achieve this, we use a $\min_{\Hv}$-$\max_{\thetav}$ formulation that combines views in a way to reveal high within-cluster variance in the combined space $\Kv_{\thetav}$ and then updates clusters by minimizing such variance. 

	The optimization problem of our method is:
	\begin{align}\label{exp:opt13}
	\mini_{\Hv \in \Rc^{n\times k}}& \maxi_{\thetav} \textbf{tr} \bb{\Kv_{\thetav} - \Hv^\top\Kv_{\thetav}\Hv}\\
	&\text{~s.t.~ } \Hv^\top\Hv = \Iv_k,~\thetav^\top\thetav \leq 1,~\thetav \geq \mathbf{0} \nonumber 
	\end{align}
	where $\Kv_{\thetav} =\sum_{v=1}^{m}\theta^{(v)}\Kv^{(v)}$. This problem can also be solved by alternately optimizing $\thetav$ and $\Hv$ given each other. 

	Note that we employ $l_2$ regularization on $\thetav$ to avoid sparse solutions. The advantages of using an $l_2$ constraint were described previously in situations when the sources of data were carefully selected and carried complementary information \citep{yu20102,kloft2009efficient,kloft2011lp}. 
	
	The $\min_{\Hv}$-$\max_{\thetav}$ framework is more favorable than $\min_{\Hv}$-$\min_{\thetav}$ in the context of multiview clustering. At every iteration, the inner maximization finds a combination of the views that maximizes within-cluster variance, while the outer minimization updates clusters that minimizes such variance. We argue that by revealing high within-cluster variance in the combined space, our method can capture more comprehensive patterns of multiple views and thus has a better opportunity to find `true' clusters. In the presence of adversarial perturbation, $\max_{\thetav}$ is particularly important because the effect of such perturbation can be mitigated when the method can tolerate a high within-cluster variance.

	In contrast, the $\min_{\Hv}$-$\min_{\thetav}$ framework aims to find a combination of the views that minimizes within-cluster variance, and then updates clusters that minimize such variance. That is, $\min_{\Hv}$-$\min_{\thetav}$ approach is not designed to tolerate the view with high within-cluster variance. This can be problematic because adversarial perturbation can cause the perturbed view(s) to have higher within-cluster variance, which makes $\min_{\Hv}$-$\min_{\thetav}$ to ignore such view(s).
	
	In Section~\ref{sec:minmax-discuss}, we illustrate with an example how adversarial perturbation affects multiview clustering and how the $\min_{\Hv}$-$\max_{\thetav}$ and $\min_{\Hv}$-$\min_{\thetav}$ frameworks behave under adversarial perturbations.
	
	\section{Algorithm}\label{sec:algorithm}
	We alternately optimize the kernel coefficients $\thetav$ and the continuous cluster assignment matrix $\Hv$ given each other: (i) given $\Hv$, $\thetav$ is optimized by solving a quadratically constrained linear programming (QCLP) problem, and (ii) given $\thetav$, $\Hv$ is optimized by solving the problem (\ref{exp:opt03}). \textsf{R} package implemented our method is freely available at \url{https://github.com/SeojinBang/MKKC}.\\
	
	Before the iteration, we center the combined mapping function $\phiv_{\thetav}\bb{\xv_i}$ by using a kernel trick $\Kv_{\thetav} \leftarrow \Kv_{\thetav}\text{ -- }\Jv_n \Kv_{\thetav}\text{ -- }\Kv_{\thetav} \Jv_n + \Jv_n \Kv_{\thetav} \Jv_n$ where $\Jv_n = \mathbf{1}_n\mathbf{1}_n^\top/n$ \citep{scholkopf1998nonlinear}. We scale each kernel matrix before combining them by ${\Kv}^{(v)} \leftarrow {\Kv^{(v)}}/\mathbf{tr} \bb{\Kv^{(v)}}$ to make multiple views comparable to each other \cite{ong2008automated,kloft2011lp}. We refer to Text \ref{centering_scaling} for a detailed discussion about centering and scaling.
	
	\subsection{Estimation of $\thetav$}\label{sec:estimationtheta} 
	Given $\Hv$, the optimization problem (\ref{exp:opt13}) is reformulated as:
	\begin{align}\label{exp:opt14}
	\maxi_{\thetav} & \sum_{v=1}^{m}{\theta^{(v)}} \textbf{tr} \bb{\Kv^{(v)}-\Hv^\top\Kv^{(v)}\Hv}\\
	&\text{~s.t.~ } 
	\frac{1}{2}\thetav^\top\Qv_m\thetav \leq 1,~\thetav \geq \mathbf{0} \nonumber 
	\end{align}
	where $\Qv_m = \text{diag}\sbb{2, \cdots, 2}$. Since $\Qv_m$ is a diagonal matrix, this problem is separable. Hence, the entire problem is solved as a conic quadratic program (i.e. second order cone program). It usually performs better than QCLP and is based on more solid duality theory \citep{mosektechreport}. Therefore, we translate QCLP to the conic formulation as follows:
	\begin{align}
	\maxi_{\thetav} &~\cv^\top\thetav\quad
	\text{~s.t.~ } \sbb{p, \thetav}^\top \in \Kc^q,~p = 1,~\mathbf{0} \leq \Iv_m\thetav\nonumber
	\end{align}
	where $\scriptstyle \cv^\top = [\textbf{tr} \bb{\Kv^{(v)}\text{-- }\Hv^\top\Kv^{(v)}\Hv}$ $\cdots$, $\scriptstyle \textbf{tr} \bb{\Kv^{(m)}\text{-- }\Hv^\top\Kv^{(m)}\Hv}]$ and $\scriptstyle \Kc^q = \cbb{p \geq \sqrt{\sum_{v=1}^{m}{\theta^{(v)}}^2}}$. This problem is analytically solved by existing software such as \texttt{mosek} \citep{mosek}. In fact, the optimization problem has a closed form solution (See Proposition~\ref{theta-closed-sol} and \ref{lem:trick2} for proof): 
	\begin{align*}
	\scriptstyle
	\thetav = \left( \tfrac{g^{(1)}}{\sqrt{\left(g^{(1)}\right)^2 + \cdots + \left(g^{(m)}\right)^2}}, \cdots, \tfrac{g^{(m)}}{\sqrt{\left(g^{(1)}\right)^2 + \cdots + \left(g^{(m)}\right)^2}} \right)
	\end{align*}
	where $g^{(v)}(\Hv) = \textbf{tr} \bb{\Kv^{(v)}-\Hv^\top\Kv^{(v)}\Hv}$ is the within-cluster variance in view $v$ and $g^{(v)}(\Hv) \geq 0$. More precisely, $\scriptstyle g^{(v)}(\Hv)$ is a sum of variance and covariance of view $v$ that are not explained by the previous clusters $\Hv$. Therefore, a view will have larger $\theta^{(v)}$ if its variability is not well explained by previous clsuters; and a combined view weighted by such $\thetav$ will have higher within-cluster variance
	by doing so, it updates $\thetav$ to find a combination of views with higher within-cluster variance.

    We mathematically prove it by showing 
	\begin{dmath*}
		\scriptstyle
		g^{(v)}(\Hv) = \textbf{tr} \bb{\Xv^{(v)}{\Xv^{(v)}}^\top} - \left( \textbf{tr} \bb{{\Vv_{1:k}^{(v)}}^\top{\Xv^{(v)}}^\top\Xv^{(v)}\Vv_{1:k}^{(v)}} + \sum_{w\neq v}\textbf{tr} \bb{{\Vv_{1:k}^{(v)}}^\top{\Xv^{(v)}}^\top\Xv^{(w)}\Vv_{1:k}^{(w)}} \right)
	\end{dmath*}
	where $\Xv^{(v)}$ is a $n\times p_v$ centered data matrix for view $v$, $\Vv_{1:k}^{(1)}$ is a matrix including the first $p_1$ rows of an orthogonal matrix $\Vv_{1:k}$ whose columns are the first $k$ right-singular vectors of $\Xv$, $\Vv_{1:k}^{(2)}$ is a matrix including the next $p_2$ rows of $\Vv_{1:k}$, and so on. Without loss of generality, we assume $\phiv(\xv) = \xv$. See Proposition~\ref{mkpca-conclusion} for proof. 
	
	This equation provides a more precise description about how $\thetav$ is estimated. Note that $\scriptstyle \textbf{tr} \bb{\Xv^{(v)}{\Xv^{(v)}}^T}$ is viewed as total variance of view $v$; $\scriptstyle \textbf{tr} \bb{{\Vv_{1:k}^{(v)}}^T{\Xv^{(v)}}^T\Xv^{(v)}\Vv_{1:k}^{(v)}}$ is viewed as variance of view $v$ explained by $\Hv$; $\scriptstyle \textbf{tr} \bb{{\Vv_{1:k}^{(v)}}^T{\Xv^{(v)}}^T\Xv^{(w)}\Vv_{1:k}^{(w)}}$ is viewed as covariance of view $v$ and view $w$ explained by $\Hv$. Consequently, above equation tells that $\scriptstyle g^{(v)}(\Hv)$ is a sum of unexplained variance and covariance of view $v$ given previous clusters $\Hv$. Considering $\theta^{(v)}$ is proportional to $\scriptstyle g^{(v)}(\Hv)$, we conclude that a view has a greater $\theta^{(v)}$ when its variability is not well explained by previous clusters.
	
	\subsection{Estimation of $\Hv$}
	Given $\thetav$, the optimization problem (\ref{exp:opt13}) is reduced to a simple kernel $k$-means clustering problem. This is the same with the problem (\ref{exp:opt03}) and the optimal solution is $\Hv = \Uv_k\Qv$. Columns of $\Uv_k$ are eigenvectors of $\Kv$ corresponding to the $k$ largest eigenvalues, and $\Qv$ is an arbitrary orthogonal matrix. Hence, any spectral clustering methods can be used to restore the binary clustering assignment matrix $\Zv$ from the continuous clustering assignment matrix $\Hv$. Here, we use a spectral clustering method proposed by \citet{ng2002spectral}.
	
	\subsection{Illustration with an example}\label{sec:minmax-discuss}
    We illustrate with an example how an adversarial feature affects multiview clustering and how the $\min_{\Hv}$-$\max_{\thetav}$ and $\min_{\Hv}$-$\min_{\thetav}$ frameworks behave under adversarial perturbation. 
    Consider a two-view data $\{(\xv_i^{(A)}, \xv_i^{(B)})\}_{i=1}^{N}$ and unobserved cluster labels $\{c_i\}_{i=1}^{N}$ ($c = 1, 2, 3$) for samples in the data where the two views, view A and view B, have complementary patterns from each other. More precisely, we consider:
	\begin{align*}
	\xv^{(A)}~|~c~\sim \mathcal{N}(\muv_1\cdot 1_{\{c = 1\}}+\muv_2\cdot 1_{\{c \neq 1\}},~\Sigma^2)\\
	\xv^{(B)}~|~c~\sim \mathcal{N}(\muv_1\cdot 1_{\{c = 3\}}+\muv_2\cdot 1_{\{c \neq 3\}},~\Sigma^2)
	\end{align*}
	where $\muv_1 \neq \muv_2$; $\xv_i^{(A)} \in \mathcal{R}^{p_A}$ and $\xv_i^{(B)}  \in \mathcal{R}^{p_B}$ where $p_A = p_B = p$ for simplicity. Note that $\xv_i^{(A)}$ can separate cluster 1 from others and $\xv_i^{(B)}$ can separate cluster 3 from others, hence views A and B together can separate all three clusters. We then add an adversarial feature, 
	\begin{align*}
	\epsilon \sim \mathcal{N}(0, 1)
	\end{align*} to view A, hence view A has $p + 1$ features, $\xv_i^{(A)}$ and $\epsilon$. Without loss of generality, we assume that $\phiv(\xv) = \xv$. We also center and scale features so that each feature has zero sample mean and unit variance. 
	
	Following \citet{ong2008automated,kloft2011lp}, we scale the kernels (discussed in Text \ref{centering_scaling}), so that the two views have the same total variance by doing $\scriptstyle {\Kv}^{(v)} \leftarrow {\Kv^{(v)}}/\mathbf{tr} \bb{\Kv^{(v)}}$. This is important because it allows the two views to be comparable to each other. Thus, we get $\scriptstyle \Kv^{(A)} = \left[\frac{1}{p+1}\left(\xv_i^{(A)}\cdot\xv_j^{(A)} + \epsilon_i\epsilon_j\right)\right]_{ij}$ and $\scriptstyle \Kv^{(B)} = \left[\frac{1}{p}\xv_i^{(B)}\cdot\xv_j^{(B)}\right]_{ij}$.
	At the initial step, with uniform kernel coefficients $\scriptstyle \theta^{(A)} = \theta^{(B)} = 1/2$, we have: 
	\begin{align}\label{eq:combinedkernel}
	\scriptstyle \Kv_{\thetav} = \scriptstyle\left[\frac{1}{p+1}\left(\xv_i^{(A)}\cdot\xv_j^{(A)} + \epsilon_i\epsilon_j\right) + \frac{1}{p}\xv_i^{(B)}\cdot\xv_j^{(B)}\right]_{ij} \times \frac{1}{2}.
	\end{align}
	Note that each element of $\Kv_{\thetav}$ represents dissimilarity between a pair of samples in the combined space. Eq.~(\ref{eq:combinedkernel}) shows that dissimilarity measured by $\xv^{(B)}$ contributes more than that measured by $\xv^{(A)}$, which can make the initial clusters that are largely based on $\xv^{(B)}$. This leads to the between-cluster variance of $\xv^{(B)}$ greater than that of $\xv^{(A)}$, i.e., $\scriptstyle \sum_{c=1}^{3} N_c\overline{\xv_c}^{(B)}\cdot\overline{\xv_c}^{(B)} \geq \sum_{c=1}^{3} N_c\overline{\xv_c}^{(A)}\cdot\overline{\xv_c}^{(A)}$ (in probability) where $\scriptstyle \overline{\xv_c} = \sum_{c_i = c}\xv_i/N_c$ is the cluster center and $N_c$ is the size of cluster $c$. 
	
	Note that within-cluster variance in each view is given as: 
	\begin{align*}
	\scriptstyle g^{(A)}(\Hv) &\scriptstyle = 1 - \frac{1}{p+1}\sum_{c=1}^{3} N_c\left(\overline{\xv}_c^{(A)}\cdot\overline{\xv}_c^{(A)} + \overline{\epsilon_c}^2\right)\\
	\scriptstyle g^{(B)}(\Hv) &\scriptstyle = 1 - \frac{1}{p}\sum_{c=1}^{3} N_c\overline{\xv_c}^{(B)}\cdot\overline{\xv}_c^{(B)}.
	\end{align*}
    From above, we can infer that $g^{(A)}(\Hv) \geq g^{(B)}(\Hv)$ (in probability) from $\scriptstyle \sum_{c=1}^{3} N_c\overline{\xv_c}^{(B)}\cdot\overline{\xv_c}^{(B)} \geq \sum_{c=1}^{3} N_c\overline{\xv_c}^{(A)}\cdot\overline{\xv_c}^{(A)}$ and $\scriptstyle \overline{\epsilon_c} = 0$ where $\scriptstyle\overline{\epsilon_c} = \sum_{c_i = c}\epsilon_i/N_c$ is the cluster mean of $\epsilon$. On the other hand, if there is no adversarial perturbation, the within-cluster variance of view A tends to be equal to that of view B, i.e., $g^{(A)}(\Hv) = g^{(B)}(\Hv)$ in probability. Thus, the adversarial feature added to view A causes the disparity of the within-cluster variance between views A and B (i.e. $\scriptstyle g^{(A)}(\Hv)$ and $\scriptstyle g^{(B)}(\Hv)$).

    Under the disparity originated from the adversarial feature added to view A, the $\min_{\Hv}$-$\max_{\thetav}$ first updates $\theta^{(A)}$ to have a larger value than $\theta^{(B)}$. This makes variance of view A is magnified (relative to view B) in the combined space so that $\Hv$ can explain more variability of view A than B. At every later iteration, $\min_{\Hv}$-$\max_{\thetav}$ alternately magnifies variance of each view while alleviating the disparity. As a result, it yields clusters at a saddle point where both views are almost equally favored, thus finding all three clusters by using complementary patterns from both views. 
    
    On the other hand, the $\min_{\Hv}$-$\min_{\thetav}$ first updates $\theta^{(B)}$ to have a larger value than $\theta^{(A)}$. This makes variance of view B is magnified (relative to view A) in the combined space so that $\Hv$ can explain more variability of view B than A. At every later iteration, $\min_{\Hv}$-$\min_{\thetav}$ keep magnifying variance of view B while aggravating the disparity. As a result, it yields clusters that are largely determined by view B.
	
	Furthermore, from this example, we can see that the $\min_{\Hv}$-$\min_{\thetav}$ framework is not preferable when: i) a view is adversarially perturbed; and/or ii) true clusters are determined by complement views that together provide comprehensive patterns about the clusters.
	
	This above-described difference between the $\min_{\Hv}$-$\max_{\thetav}$ and $\min_{\Hv}$-$\min_{\thetav}$ is also observed in simulation experiments (Figure~\ref{fig:thetavsiter}). 
	\begin{figure}[!ht]
		\centerline{\includegraphics[width=\linewidth]{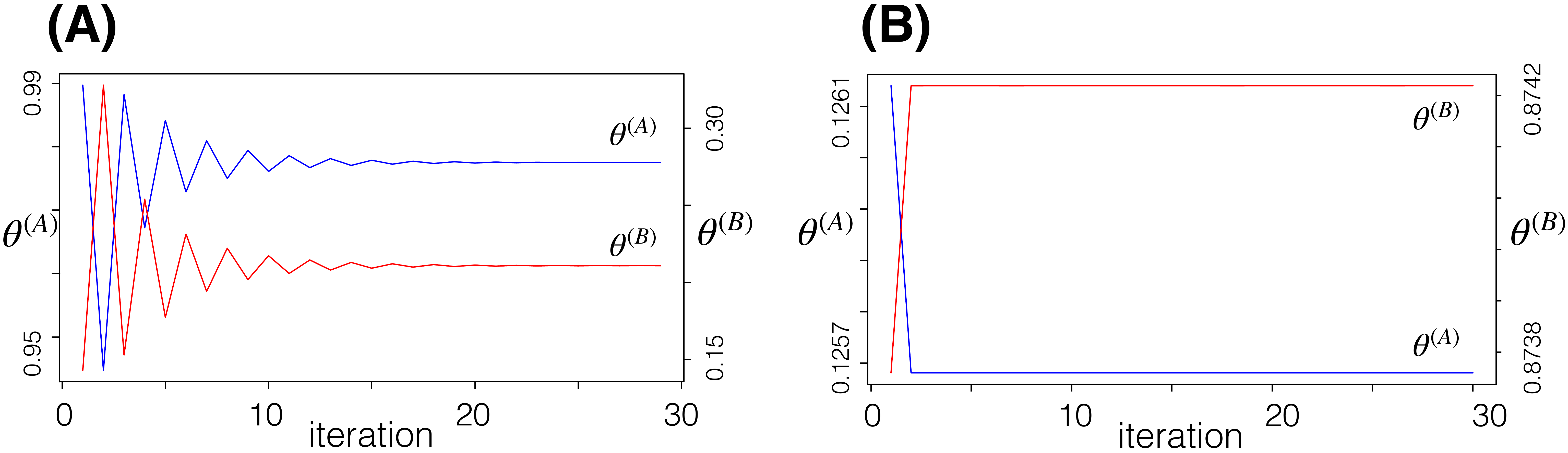}}
		\captionsetup{font=footnotesize,labelfont=footnotesize}
		\caption{Kernel coefficients $\thetav$ are updated by (A) the $\min_{\Hv}$-$\max_{\thetav}$ (our method) and (B) the $\min_{\Hv}$-$\min_{\thetav}$ (Gonen's MKK is shown here, but other compared methods show similar behaviors). One of our simulation data (B-Noise with five noise variables) is used.}\label{fig:thetavsiter}
	\end{figure}
	
	\subsection{Proof of convergence}\label{proofconv}
	We prove that if our alternating strategy converges, it will converge to the global optimal solution, which motivates our alternating strategy. First, recall the $\min_{\Hv}$-$\max_{\thetav}$ problem (\ref{exp:opt13}) and the optimization function $f(\Hv, \thetav) = \textbf{tr} \bb{\Kv_{\thetav} - \Hv^\top\Kv_{\thetav}\Hv}$ where $\Hv^\top\Hv = \Iv_k,~\thetav^\top\thetav \leq 1$, and $\thetav \geq \mathbf{0}$. A saddle point $(\Hv^*, \thetav^*)$ of this problem is defined as follows: 
	\begin{align*}
	&f(\Hv^*, \thetav^*) \geq f(\Hv^*, \thetav) \text{ for all } \thetav\\
	&f(\Hv^*, \thetav^*) \leq f(\Hv, \thetav^*) \text{ for all } \Hv.
	\end{align*} 
	That is, given $\Hv^*$, $\thetav^*$ is the maximum among all $\thetav$ and given $\thetav^*$, $\Hv^*$ is the minimum among all $\Hv$. From this definition, it is clear that if the alternating strategy converges, then it converges to a saddle point. Moreover, it is known that if $(\Hv^*, \thetav^*)$ is a saddle point, then: (1) $\Hv^*$ is a globally optimal solution for  $\min_{\Hv} f_1(\Hv)$ where $f_1(\Hv) = \max_{\thetav} f(\Hv, \thetav)$ and (2) $\thetav^*$ is a globally optimal solution for $\max_{\thetav} f_2(\thetav)$ where $f_2(\thetav) = \min_{\Hv} f(\Hv, \thetav)$. Therefore, if we want to find a global minimizer $\Hv$, we can try to find a saddle point $(\Hv^*, \thetav^*)$ whose definition motivates our alternating strategy. This alternating strategy is practically efficient and easy to implement because each of the two steps ($\min_{\Hv}$, $\max_{\thetav}$) has a closed-form solution which requires fewer iterations than a gradient approach to converge. Such alternating approach has also been used in other works such as Generative Adversarial Net \citep{goodfellow2014generative}, and robust models against adversarial examples \citep{madry2018towards,sinha2018certifying}.
	
	\section{Simulation experiments}\label{sec:simulation}
	\subsection{Adversarial perturbations}
	We evaluate robustness of our method against two types of adversarial features added to a view:
	\begin{itemize}
		\item Noise variables that are independently sampled from Gaussian distribution with zero-mean and unit-variance. We add different numbers ($N_{noise} = 0, 1, \cdots$) of noise variables.
		\item Redundant variables that are correlated with original variables. We add different numbers ($N_{redun} = 1, 2, \cdots$) of variables having different correlations ($cor = 1, 0.97, 0.90, 0.72, 0.45$) with the original variables.
	\end{itemize}
	Under these perturbations, we examine how our method makes use of complementary patterns of multiple views. For this purpose, we first generated multiview data in three scenarios A--C having two or three views that have complementary patterns necessary for identifying true clusters. Scenario A is composed of a complete view that has complete information to detect the three clusters and a partial view that only conveys partial information. Scenario B is composed of two different partial views so that each view alone cannot completely detect the three clusters. Both scenarios A \& B aim to test how the compared methods use the complementary information in two views. Scenario C is composed of two different partial views and a noise view. It aims to test further whether the methods robustly use complementary information from views even when one of the views contains only noise variables.
	
	Then, as illustrated in Figure~\ref{fig:multiviewscenAll}, we added different types and levels of adversarial features to one of the views. We denote the simulation data with the noise variables by A-Noise, B-Noise, and C-Noise, and the data with the redundant variables by A-Redun, B-Redun, and C-Redun. 

	For detailed explanation about data and preprocessing, see Text \ref{sim-detail}. 
	\begin{figure}[!ht]
		\centering
		\captionsetup{font=footnotesize,labelfont=footnotesize}
		\includegraphics[width=\columnwidth]{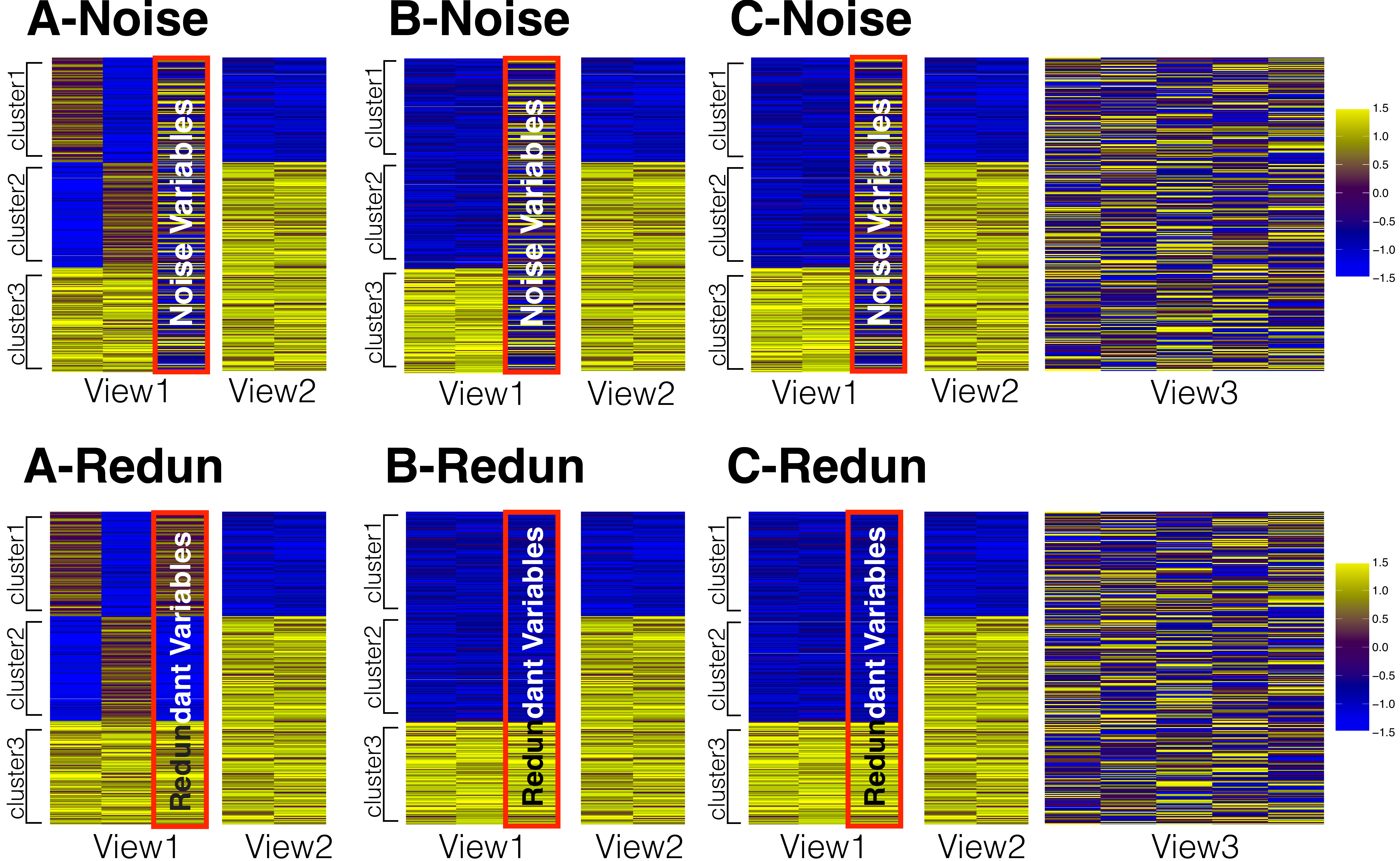}
		\caption{Simulation scenarios with two types of adversarial perturbations, noise and redundancy. The samples in the heatmaps are ordered by their true cluster index. Each simulation data contains 300 samples with 100 samples in each cluster, i.e. three true clusters for each data.
		}\label{fig:multiviewscenAll}
	\end{figure}
	
	\subsection{Compared methods}
	We compare our method with seven other methods: two baseline methods, \textbf{Single Best} and \textbf{Uniform Weight}; four multiple kernel $k$-means clustering methods, \textbf{Gonen's MKK} and \textbf{LMKK} \citep{gonen2014localized}, \textbf{Liu's MKK-MIR} \citep{liu2016multiple}, and \textbf{Yu's OKKC} \citep{yu2012optimized}); and one variant of our method, \textbf{MinMax-MinC}. Single Best uses the best view that has the smallest within-cluster variance. Uniform Weight gives the same weights to all views. Gonen's MKK, Gonen's LMKK, and Liu's MKK-MIR combine multiple kernels by $\Kv_{\thetav} = \sum_{v=1}^{m}{\theta^{(v)}}^2\Kv^{(v)}$, with $l_1$ constraint on $\thetav$, and use the $\min_{\Hv}$-$\min_{\thetav}$ framework. Yu's OKKC combines multiple kernels by $\Kv_{\thetav} = \sum_{v=1}^{m}\theta^{(v)}\Kv^{(v)}$ and uses $l_p$ constraint on $\thetav$ where $p \geq 1$, and uses the $\max_{\Hv}$-$\max_{\thetav}$ framework. MinMax-MinC is the $l_1$-regularization version of our method, which is included to examine the effect of $l_2$-regularization in our method. For a detailed description, see Text \ref{sim-detail}. 
	
	We evaluate how robustly the methods recover the true cluster against such perturbations. For evaluation, we use three metrics: Adjusted Rand Index (AdjRI, \citep{hubert1985comparing}), Normalized Mutual Information (NormMI, \citep{strehl2002cluster}), and Purity \citep{Manning2008information}.
	
	\subsection{Simulation results}
	\begin{figure*}[!ht]
		\centerline{\includegraphics[width=\textwidth]{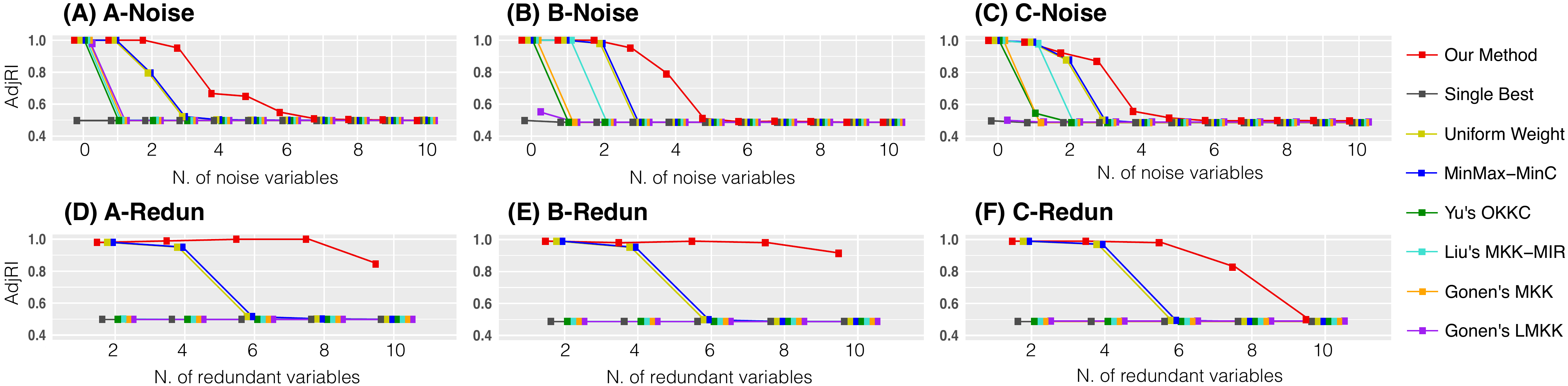}}
		\captionsetup{font=footnotesize,labelfont=footnotesize}
		\caption{Clustering performance. AdjRI versus the number of the noise (A--C) or redundant variables (D--F, $cor = 0.90$) added to view 1. The identified clusters are compared to the true clusters. }\label{fig:evalall}
	\end{figure*}
	
	\begin{figure*}[!ht]
		\centerline{\includegraphics[width=\textwidth]{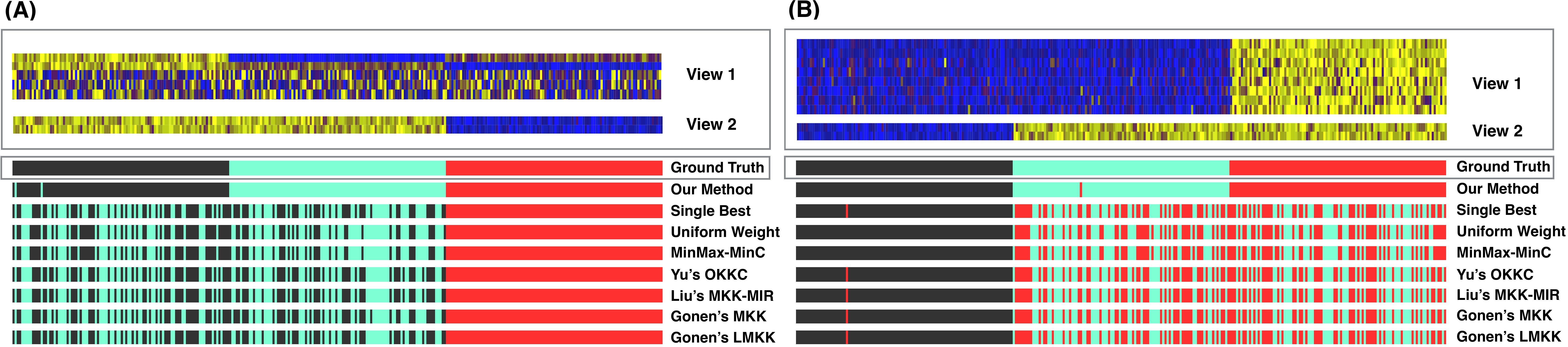}}
		\captionsetup{font=footnotesize,labelfont=footnotesize}
		\caption{Clustering results of (A) A-Noise with three noise variables and (B) B-Redun with three redundant variables and $cor = 0.7$. The heatmaps at the top panel illustrate the two-view data where the rows and columns represent variables and samples, respectively. The tile plots at the bottom panel illustrate the clustering results. The first tile plot shows the ground truth; the second to the last plots show results of the methods. Identified clusters are labeled in different colors. 
		}\label{fig:scatter}
	\end{figure*}
	
	We illustrate the results in Figures~\ref{fig:evalall} \&~\ref{fig:scatter}. As shown in Figure~\ref{fig:evalall}, when there is no adversarial perturbation, all methods accurately identifies the true clusters; however, when the adversarial perturbations are present, our method identifies the true clusters more robustly than others. Note that existing $\min_{\Hv}$-$\min_{\thetav}$ methods perform even worse than Uniform Weight. Figure~\ref{fig:scatter} illustrates that our result agrees well with the ground truth, which indicates that it better uses the complementary information in both views. However, the other $\min_{\Hv}$-$\min_{\thetav}$ methods fail to distinguish the first two clusters in Figure \ref{fig:scatter}A and the last two clusters in Figure \ref{fig:scatter}B, which suggests that they identify clusters mainly based on View 2 that is not perturbed. For further results, see Table~\ref{table:evalclusteringAnoise}--~\ref{table:evalclusteringCredun}.
	
	\section{Real data analysis}\label{sec:realdata}
	With the advent of various genome-wide technologies, a wide array of biomedical data has been available, which includes clinical characteristics, DNA copy number, and gene expression profiling. They contain complementary information that can together provide a comprehensive understanding and novel insight into biomedical problems. In this section, we present application of our method to a biomedical problem to identify cancer subtypes.
	
	\subsection{Identification of cancer subtype}
	We compared our method with other methods using two TCGA multi-omics cancer datasets. Each dataset includes 468 patients with human breast invasive carcinoma (BRCA) and 251 patients with glioblastoma multiforme (GBM), respectively \citep{weinstein2013cancer}. BRCA has three views: mRNA sequencings, miRNA sequencings, and copy number variations. GBM also has three views: gene expression microarray profiling, copy number variation, and methylation data. A radial basis function kernel is used for all views as suggested by \citet{lanckriet2004statistical}. See Text \ref{real-detail} for details about data preprocessing.
	
	Since there is no ground-truth subtype, we compared clinical properties of identified clusters that are observed independently from data and examined how distinct they are. In BRCA, we compared the \textit{AJCC neoplasm disease stage} (which describes the extent of both malignant and benign growths). In GBM, we compared the survival time (days to death) and the \textit{Karnofsky performance score} (a patient's prognosis by measuring a patient's ability to function). 
	
	Further, to help understand biological mechanisms underlying the clusters, we identified differentially expressed genes (DEGs) for each cluster. We used RNA sequencing data for BRCA and gene expression microarray data for GBM and performed the two-sample t-test. The p-values are adjusted using a Benjamini-Hochberg procedure to address multiple hypothesis testing problems \citep{benjamini1995controlling}. We performed gene set enrichment analysis \citep{subramanian2005gene} to find out the KEGG pathways enriched among the DEGs in each cluster. Then, we compared these enriched pathways with the BRCA- or GBM-related biological pathways provided by the KEGG Pathway Database (\url{https://www.kegg.jp}) that are defined independently from data (See Table \ref{table:pathwaylist}).
	
	\begin{table}[!ht]
		\caption[caption]{\textbf{Evaluation of clustering results.}} \label{table:evalclusteringBRCA}\label{table:evalclusteringGBM}
		\centering
		\begin{adjustbox}{max width=\linewidth}
			\def\arraystretch{0.95}
			\begin{threeparttable}
				\begin{tabular}{@{}lc|cc@{}}
					\toprule
					\small Cancer Type & \multicolumn{1}{c}{\small \textbf{BRCA}} & \multicolumn{2}{c}{\small \textbf{GBM}} \\ \midrule
					\small Method &  \begin{tabular}{@{}c@{}}\small AJCC disease stage \\ \small (p-value)\end{tabular} &  \begin{tabular}{@{}c@{}} \small Survival time \\ \small (p-value) \end{tabular}  & \begin{tabular}{@{}c@{}} \small Karnofsky score \\ \small (p-value) \end{tabular}  \cr \midrule 
					\small Single Best 		& 0.54						& $1.13\times 10^{-1}$			& 0.89	\cr
					\small Uni. Weight 		& 0.22          			& $\mathbf{1.37\times 10^{-5}}$	& 0.13	\cr
					\small MinMax-MinC		& 0.22			        	& $\mathbf{1.37\times 10^{-5}}$	& 0.13	\cr
					\small Yu's OKKC		& 0.53						& $2.51\times 10^{-5}$			& 0.26	\cr
					\small Liu's MKK-MIR 	& 0.42						& $2.21\times 10^{-5}$			& 0.16	\cr
					\small Gonen's MKK 		& 0.48						& $1.46\times 10^{-3}$			& 0.88	\cr
					\small Gonen's LMKK 	& 0.56 						& $6.68\times 10^{-2}$  		& 0.74	\cr
					\small MML-MKKC  		& \textbf{0.09}				& $\mathbf{1.37\times 10^{-5}}$	& \textbf{0.02}	\cr \bottomrule 
				\end{tabular}%
				\begin{tablenotes}
					\footnotesize
					\item Note: for BRCA, differences in the AJCC neoplasm disease stages 
					among the clusters are compared using the chi-square test. For GBM, differences in the survival curves and the Karnofsky performance scores among the clusters are compared using the log-rank test \citep{mantel1966evaluation} and the chi-square test, respectively.
				\end{tablenotes}
			\end{threeparttable}
		\end{adjustbox}	
	\end{table}
	
	\textbf{Results:} For \textbf{\textit{BRCA}}, we identified five clusters (92, 86, 83, 137, and 70 subjects for each cluster) using each methods. Table~\ref{table:evalclusteringBRCA} shows that our method identified clusters that have the most distinct disease stages. Further, our method identified clusters that have many enriched pathways such as \textit{cell cycle} and \textit{alanine} that are consistent with the BRCA-related pathways in the KEGG pathway database (See Table~\ref{table:pathwayBRCA}). For \textbf{\textit{GBM}}, we identified five clusters (58, 38, 39, 46, and 70 subjects for each cluster). Table~\ref{table:evalclusteringGBM} shows that our method identified clusters that have the most distinct survival time and the patients' ability to carry daily activities (as measured by the Karnofsky score). Further, our method identified clusters that have many enriched pathways such as \textit{adherens junction} and \textit{calcium signaling pathway} that are consistent with the GBM-related pathways in the KEGG pathway database (See Table~\ref{table:pathwayGBM}). Together, these results show that, compared to the other methods, our method better identifies distinct clusters that are relevant to the pathobiological mechanisms underlying the diseases.
	
	%
		
	\section{Conclusion}\label{sec:conclusion}
	In this paper, we investigate the effects of adversarial perturbation on multiple kernel k-means clustering. We show that such perturbation can make the existing methods with the $\min_{\Hv}$-$\min_{\thetav}$ formulation ignore the perturbed view and find clusters largely depend on other view(s). To address this problem, we propose a multiple kernel k-means clustering method, MML-MKKC, which aims to be robust to adversarial perturbation by using the $\min_{\Hv}$-$\max_{\thetav}$ formulation. 
	
	Our algorithm is practically efficient and easy to implement because it alternately optimizes $\thetav$ and $\Hv$ where each of the two steps ($\min_{\Hv}$ and $\max_{\thetav}$) has a closed-form solution that requires fewer iterations than a gradient approach to converge. In simulation experiments, we showed that our method is more robust to adversarial perturbation than other methods. In real data analysis, our method identified the most distinct clusters of cancer patients.
	
	
\bibliographystyle{natbib}
\bibliography{references.bib}

\clearpage
\beginsupplement
    \section*{Supplementary Materials}
    \begin{theorem} \label{theo:kyfan}
	\cite{fan1949theorem} Let $\Kv$ is a symmetric matrix where $\uv_1, \cdots, \uv_n$ are eigenvectors corresponding to eigenvalues $\lambda_1 \geq \cdots \geq \lambda_n$ of $\Kv$. Then, the optimal solution of the problem
	\begin{align*}
	\argmax_{\Hv \in \Rc^{n\times k}} \textbf{tr}\bb{\Hv^T\Kv\Hv}\\
	\text{subject to } 
	& \Hv^T\Hv = \Iv_k
	\end{align*}
	is given by $\Hv^* = \Uv_k\Qv$ where $\Uv_k = [\uv_1, \cdots, \uv_k]$ and $Q$ is an arbitrary $k\times k$ orthogonal matrix, and the maximum is given by 
	\begin{align*}
	\max_{\Hv \in \Rc^{n\times k}} \textbf{tr}\bb{\Hv^T\Kv\Hv} = \textbf{tr}\bb{{\Hv^*}^T\Kv\Hv^*} = \sum_{i = 1}^{k} \lambda_i
	\end{align*}
\end{theorem}
\bigskip
    ~\newpage
    \begin{proposition} \label{mkpca}
Suppose that $\Xv^{(v)}$ is a $n\times p_v$ centered data matrix from $n$ samples and $p_v$ random variables, that $\Xv = \sbb{\Xv^{(1)}, \cdots, \Xv^{(m)}}$ is an $n\times p$ multiview data matrix collected from $m$ multiple sources where $p = p_1 + \cdots + p_m$, and that $\Kv^{(v)} = \Xv^{(v)}{\Xv^{(v)}}^T$. Then,
\begin{align*}
    \textbf{tr} \bb{\Hv^T\Kv^{(v)}\Hv} = \textbf{tr} \bb{{\Vv_{1:k}^{(v)}}^T{\Xv^{(v)}}^T\Xv^{(v)}\Vv_{1:k}^{(v)}} + \sum_{w\neq v}\textbf{tr} \bb{{\Vv_{1:k}^{(v)}}^T{\Xv^{(v)}}^T\Xv^{(w)}\Vv_{1:k}^{(w)}}
\end{align*}
where $\Hv = \Uv_k\Qv$ is a $n\times k$ matrix where the column of $\Uv_k$ contains the first $k$ eigenvectors of $\Xv\Xv^T$ corresponding to $k$ largest eigenvalues, $\Qv$ is an arbitrary orthogonal matrix, and ${\Vv_{1:k}^{(v)}}^T$ is a $k \times p_v$ matrix including the top $k$ rows of ${\Vv^{(v)}}^T$ where ${\Vv^{(1)}}^T$ is a $p\times p_1$ matrix including the first $p_1$ eigenvector of $\Xv^T\Xv$, ${\Vv^{(2)}}^T$ is a $p\times p_2$ matrix including the next $p_2$ eigenvector of $\Xv^T\Xv$ and so on. 
\end{proposition}
\begin{proof}~\\~\\
We consider the singular value decomposition of the $n\times p$ matrix $\Xv$:
\begin{align*}
    \Xv = \Uv\Dv\Vv^T
\end{align*}
where $\Uv$ is an $n\times n$ orthogonal matrix whose columns are the the left-singular vectors of $\Xv$, $\Dv$ is an $n\times p$ rectangular diagonal matrix with non-negative real values $\sigma_1, \cdots, \sigma_p$ known as singular values of $\Xv$ on the diagonal, and $\Vv$ is an $p\times p$ orthogonal matrix whose columns are the right-singular vectors of $\Xv$. Without loss of generality, we assume the singular values of $\Xv$, $\sigma_1 \geq \cdots \geq \sigma_p$, are ordered by largest to smallest and so the corresponding column vectors of $\Uv$ and $\Vv$ are as well.\\~\\
Note that the $c$-th column vector of $\Vv$ (i.e. $c$-th right-singular vector of $\Xv$) is equivalent to the $c$-th eigenvector corresponding to the $c$-th largest eigenvalue $\lambda_c$ of $\Xv^T\Xv$. 
Then we get:
\begin{align}
    &~\Xv = \Uv\Dv\Vv^T\nonumber\\  
	\Leftrightarrow~&~\sbb{\Xv^{(1)}, \cdots, \Xv^{(m)}} = \Uv\Dv\sbb{{\Vv^{(1)}}^T, \cdots, {\Vv^{(m)}}^T} \nonumber\\ 
    \Leftrightarrow~&~\Xv^{(v)} = \Uv\Dv{\Vv^{(v)}}^T \text{~~~~~~for~} v = 1, \cdots, m \nonumber\\
    \Leftrightarrow~&~\Uv^T\Xv^{(v)} = \Dv{\Vv^{(v)}}^T \text{~~~~~~for~} v = 1, \cdots, m \label{mkpca_trick1}
\end{align}
Note that the $c$-th column vector of $\Uv$ (i.e. $c$-th left-singular vector of $\Xv$) is equivalent to the $c$-th eigenvector corresponding to the $c$-th largest eigenvalue $\lambda_c$ of $\Xv\Xv^T$. Hence, $\Uv = \sbb{\Uv_k, \Uv_C} \nonumber$ where the columns of $\Uv_k$ and $\Uv_C$ are the first $k$ and the last $n-k$ eigenvectors of $\Xv\Xv^T$ respectively.\\~\\
From (\ref{mkpca_trick1}), we get
\begin{align}\label{mkpca_trick2}
    \Uv_k^T\Xv^{(v)} = \Dv_k{\Vv_{1:k}^{(v)}}^T \text{~~~~~~for~} v = 1, \cdots, m
\end{align}
where $\Dv_k$ is a $k \times k$ diagonal matrix with the first $k$ singular values $\sigma_1, \cdots, \sigma_k$ on the diagonal and ${\Vv_{1:k}^{(v)}}^T$ is a $k \times p_v$ matrix including the top $k$ rows of ${\Vv^{(v)}}^T$.\\~\\
Using (\ref{mkpca_trick2}), we get
\begin{align*}
	\Hv^T\Kv^{(v)}\Hv & = \Hv^T\Xv^{(v)}{\Xv^{(v)}}^T\Hv\\
	& = \Qv^T\Uv_k^T\Xv^{(v)}{\Xv^{(v)}}^T\Uv_k\Qv \\
	& = \Qv^T\Dv_k{\Vv_{1:k}^{(v)}}^T\Vv_{1:k}^{(v)}\Dv_k^T\Qv 
\end{align*}
Therefore,
\begin{align*}
    \textbf{tr} \bb{\Hv^T\Kv^{(v)}\Hv} = \textbf{tr} \bb{\Qv^T\Dv_k{\Vv_{1:k}^{(v)}}^T\Vv_{1:k}^{(v)}\Dv_k^T\Qv}
\end{align*}
By the invariant property under cyclic permutations of the trace function and the fact that $\Qv$ is an orthogonal matrix, we get
\begin{align}\label{mkpca4}
    \textbf{tr} \bb{\Hv^T\Kv^{(v)}\Hv} & = \textbf{tr} \bb{\Vv_{1:k}^{(v)}\Dv_k^T\Qv\Qv^T\Dv_k{\Vv_{1:k}^{(v)}}^T}\nonumber\\
    & = \textbf{tr} \bb{\Vv_{1:k}^{(v)}\Dv_k^T\Dv_k{\Vv_{1:k}^{(v)}}^T}
\end{align}
Note that the $c$-th column vector of $\Vv$ (i.e. $c$-th right-singular vector of $\Xv$) is equivalent to the $c$-th eigenvector corresponding to the $c$-th largest eigenvalue $\lambda_c$ of $\Xv^T\Xv$ (i.e. $\sigma_c^2 = \lambda_c$). Therefore, we have:
\begin{align*}
	\Xv^T\Xv\vv_c = \lambda_c\vv_c
\end{align*}
for $c = 1, \cdots, p$. Note that the $c$-th diagonal element of $\Dv$ is equivalent to the square roots of the $c$-th eigenvalue of $\Xv^T\Xv$. Therefore, we have:
\begin{align}\label{mkpca3}
	&\Xv^T\Xv\Vv_{1:k} = \Vv_{1:k}\Dv_k^T\Dv_k \nonumber\\
	\Rightarrow~&~\Xv^T\Xv\Vv_{1:k}\Vv_{1:k}^T = \Vv_{1:k}\Dv_{1:k}^T\Dv_{1:k}\Vv_{1:k}^T 
\end{align}
The light-hand-side of (\ref{mkpca3}) has $p_v\times p_v$ square matrices (blocks) in the main diagonal as follow:
\begin{align*}
    \Vv_{1:k}\Dv_{1:k}^T\Dv_{1:k}\Vv_{1:k}^T & = \sbb{\Dv_k{\Vv_{1:k}^{(1)}}^T, \cdots, \Dv_k{\Vv_{1:k}^{(m)}}^T}^T\sbb{\Dv_k{\Vv_{1:k}^{(1)}}^T, \cdots, \Dv_k{\Vv_{1:k}^{(m)}}^T} \\
    & = \begin{bmatrix}
       \Vv_{1:k}^{(1)}\Dv_k^T\Dv_k{\Vv_{1:k}^{(1)}}^T & ~ & \text{off-diagonal entries}\\[0.3em]
       ~    & \ddots & ~\\[0.3em]
       \text{off-diagonal entries}    & ~ & \Vv_{1:k}^{(m)}\Dv_k^T\Dv_k{\Vv_{1:k}^{(m)}}^T
     \end{bmatrix}
\end{align*}
The left-hand-side of (\ref{mkpca3}) has $p_v\times p_v$ square matrices (blocks) in the main diagonal as follow:
\begin{align*}
    & \Xv^T\Xv\Vv_{1:k}\Vv_{1:k}^T\\
    & = \begin{bmatrix}
       \scriptstyle{{\Xv^{(1)}}^T\Xv^{(1)}\Vv_{1:k}^{(1)}{\Vv_{1:k}^{(1)}}^T + \sum_{w\neq 1} {\Xv^{(1)}}^T\Xv^{(w)}\Vv_{1:k}^{(w)}{\Vv_{1:k}^{(1)}}^T} & ~ & \text{off-diagonal entries}\\[0.3em]
       ~    & \ddots & ~\\[0.3em]
       \text{off-diagonal entries}    & ~ & \scriptstyle{{\Xv^{(m)}}^T\Xv^{(m)}\Vv_{1:k}^{(m)}{\Vv_{1:k}^{(m)}}^T + \sum_{w\neq m} {\Xv^{(m)}}^T\Xv^{(w)}\Vv_{1:k}^{(w)}{\Vv_{1:k}^{(m)}}^T}
     \end{bmatrix}
\end{align*}
Therefore, we have
\begin{align*}
	&{\Xv^{(v)}}^T\Xv^{(v)}\Vv_{1:k}^{(v)}{\Vv_{1:k}^{(v)}}^T + \sum_{w \neq v} {\Xv^{(v)}}^T\Xv^{(w)}\Vv_{1:k}^{(w)}{\Vv_{1:k}^{(v)}}^T = \Vv_{1:k}^{(v)}\Dv_k^T\Dv_k{\Vv_{1:k}^{(v)}}^T \\
	\Rightarrow~&~\textbf{tr} \bb{{\Vv_{1:k}^{(v)}}^T{\Xv^{(v)}}^T\Xv^{(v)}\Vv_{1:k}^{(v)}} + \sum_{w \neq v} \textbf{tr} \bb{{\Vv_{1:k}^{(v)}}^T{\Xv^{(v)}}^T\Xv^{(w)}\Vv_{1:k}^{(w)}} = \textbf{tr} \bb{\Vv_{1:k}^{(v)}\Dv_k^T\Dv_k{\Vv_{1:k}^{(v)}}^T}
\end{align*}
for $v = 1, \cdots, m$. From (\ref{mkpca4}), that is:
\begin{align*}
    \textbf{tr} \bb{\Hv^T\Kv^{(v)}\Hv} = \textbf{tr} \bb{{\Vv_{1:k}^{(v)}}^T{\Xv^{(v)}}^T\Xv^{(v)}\Vv_{1:k}^{(v)}} + \sum_{w \neq v} \textbf{tr} \bb{{\Vv_{1:k}^{(v)}}^T{\Xv^{(v)}}^T\Xv^{(w)}\Vv_{1:k}^{(w)}}
\end{align*}
\end{proof}
\bigskip

    ~\newpage
    \begin{proposition} \label{mkpca-conclusion}
Suppose that $\Xv^{(v)}$ is a $n\times p_v$ centered data matrix from $n$ samples and $p_v$ random variables, that $\Xv = \sbb{\Xv^{(1)}, \cdots, \Xv^{(m)}}$ is an $n\times p$ multiview data matrix collected from $m$ multiple sources where $p = p_1 + \cdots + p_m$, and that $\Kv^{(v)} = \Xv^{(v)}{\Xv^{(v)}}^T$. Then,
\begin{align*}
    \textbf{tr} \bb{\Kv^{(v)}-\Hv^T\Kv^{(v)}\Hv} =~&~\textbf{tr} \bb{\Xv^{(v)}{\Xv^{(v)}}^T} \\
    &- \left( \textbf{tr} \bb{{\Vv_{1:k}^{(v)}}^T{\Xv^{(v)}}^T\Xv^{(v)}\Vv_{1:k}^{(v)}} + \sum_{w\neq v}\textbf{tr} \bb{{\Vv_{1:k}^{(v)}}^T{\Xv^{(v)}}^T\Xv^{(w)}\Vv_{1:k}^{(w)}} \right)
\end{align*}
where $\Hv = \Uv_k\Qv$ is a $n\times k$ matrix where the column of $\Uv_k$ contains the first $k$ eigenvectors of $\Xv\Xv^T$ corresponding to $k$ largest eigenvalues, $\Qv$ is an arbitrary orthogonal matrix, and ${\Vv_{1:k}^{(v)}}^T$ is a $k \times p_v$ matrix including the top $k$ rows of ${\Vv^{(v)}}^T$ where ${\Vv^{(1)}}^T$ is a $p\times p_1$ matrix including the first $p_1$ eigenvector of $\Xv^T\Xv$, ${\Vv^{(2)}}^T$ is a $p\times p_2$ matrix including the next $p_2$ eigenvector of $\Xv^T\Xv$ and so on. 

Hence, $\textbf{tr} \bb{\Kv^{(v)}-\Hv^T\Kv^{(v)}\Hv}$ can be interpreted as the unexplained variability (both the unexplained variance of the view $v$ and the unexplained covariance of the view $v$ with the other views $w$) by principal components of the joint representation.
\end{proposition}
\begin{proof}~\\~\\
Suppose that $\Phiv^{(v)}$ is a $n \times d_v$ centered data matrix from $n$ samples and $d_v$ random variables in the nonlinear feature space $\Fc$, and that $\Phiv = \sbb{\Phiv^{(1)}, \cdots, \Phiv^{(v)}}$ is a $n\times d$ multiview data matrix collected from multiple sources $v = 1, \cdots, m$ where $d = d_1 + \cdots + d_m$ and $m$ is the number of views.

By Proposition \ref{mkpca}, we knows 
\begin{align*}
    \textbf{tr} \bb{\Hv^T\Kv^{(v)}\Hv} = \textbf{tr} \bb{{\Vv_{1:k}^{(v)}}^T{\Xv^{(v)}}^T\Xv^{(v)}\Vv_{1:k}^{(v)}} + \sum_{w\neq v}\textbf{tr} \bb{{\Vv_{1:k}^{(v)}}^T{\Xv^{(v)}}^T\Xv^{(w)}\Vv_{1:k}^{(w)}}
\end{align*}
And hence,
\begin{align*}
    \textbf{tr} \bb{\Kv^{(v)}-\Hv^T\Kv^{(v)}\Hv} =~&~\textbf{tr} \bb{\Xv^{(v)}{\Xv^{(v)}}^T} \\
    &- \left( \textbf{tr} \bb{{\Vv_{1:k}^{(v)}}^T{\Xv^{(v)}}^T\Xv^{(v)}\Vv_{1:k}^{(v)}} + \sum_{w\neq v}\textbf{tr} \bb{{\Vv_{1:k}^{(v)}}^T{\Xv^{(v)}}^T\Xv^{(w)}\Vv_{1:k}^{(w)}} \right)
\end{align*}
where $\scriptstyle \textbf{tr} \bb{\Xv^{(v)}{\Xv^{(v)}}^T}$ is the total variance of the view $v$; 
$\scriptstyle \textbf{tr} \bb{{\Vv_{1:k}^{(v)}}^T{\Xv^{(v)}}^T\Xv^{(v)}\Vv_{1:k}^{(v)}}$ is the variance of the view $v$ explained by the eigenvectors of the feature space; $\scriptstyle \textbf{tr} \bb{{\Vv_{1:k}^{(v)}}^T{\Xv^{(v)}}^T\Xv^{(w)}\Vv_{1:k}^{(w)}}$ is the covariance of the view $v$ with the view $w$ explained by the eigenvectors of the feature space. 
\end{proof}
\bigskip

    ~\newpage 
    \begin{proposition} \label{theta-closed-sol}
The optimization problem 
\begin{align}
\maxi_{\thetav} & \sum_{v=1}^{m}{\theta^{(v)}} \textbf{tr} \bb{\Kv^{(v)}-\Hv^T\Kv^{(v)}\Hv}\\
&\text{subject to } 
\frac{1}{2}\thetav^T\Qv_m\thetav \leq 1,~\thetav \geq \mathbf{0} \nonumber
\end{align}
has a closed form solution 
\begin{align*}
    \thetav = \left( \frac{a^{(1)}}{\sqrt{\left(a^{(1)}\right)^2 + \cdots + \left(a^{(m)}\right)^2}}, \cdots, \frac{a^{(m)}}{\sqrt{\left(a^{(1)}\right)^2 + \cdots + \left(a^{(m)}\right)^2}} \right)
\end{align*}
where $a^{(v)} = \textbf{tr} \bb{\Kv^{(v)}-\Hv^T\Kv^{(v)}\Hv}$ for $v = 1, \cdots, m$, $\Hv$ is a real valued $n\times k$ matrix $\Hv$ such that $\Hv^T\Hv = \Iv_k$, and $\Kv^{(v)}$ is a $n\times n$ positive semidefinite matrix.
\end{proposition}
\begin{proof}~\\~\\
We solve the optimization problem with geometric perspective. First, we define an optimal plane:
\begin{align}\label{plane}
    k = \theta^{(1)}a^{(1)} + \cdots + \theta^{(m)}a^{(m)}
\end{align}
The plane is restricted to one of those that touch or pass through the hypersphere $\left(\theta^{(1)}\right)^2 + \cdots + \left(\theta^{(m)}\right)^2 = 1$ where $\thetav \geq 0$ and $a^{(v)}$s have non-negative real values by Proposition \ref{lem:trick2}. The optimal solution is obtained at which the plane maximizes $k > 0$. In order to obtain the tangent point, we first find the normal vector that is perpendicular to the surface of the plane and passes through the center of the hypersphere (i.e. $\thetav = \mathbf{0}$) as follow:
\begin{align*}
    \frac{\theta^{(1)}-0}{a^{(1)}} = \cdots = \frac{\theta^{(m)}-0}{a^{(m)}} = t
\end{align*}
which is equivalent to
\begin{align*}
    \thetav(t) = \left( a^{(1)}t, \cdots, a^{(m)}t \right)
\end{align*}
where $t$ is a real valued scalar variable. The plane touches the hypersphere at which the normal vector passing through the surface of the hypersphere.
\begin{align*}
    &~\left(\theta^{(1)}\right)^2 + \cdots + \left(\theta^{(m)}\right)^2 = 1 \\
    \Leftrightarrow~&~\left(a^{(1)}t\right)^2 + \cdots + \left(a^{(m)}t\right)^2 = 1\\
    \Leftrightarrow~&~\left(a^{(1)}\right)^2 + \cdots + \left(a^{(m)}\right)^2 = \frac{1}{t^2}
\end{align*}

(\rnum{1}) In order for the plane to be tangent to the hypersphere where $\thetav \geq 0$, $t$ should be a positive real value, hence, we get
\begin{align*}
    t = \frac{1}{\sqrt{\left(a^{(1)}\right)^2 + \cdots + \left(a^{(m)}\right)^2}}
\end{align*}

(\rnum{2}) In order for the plane to pass through the hypersphere where $\thetav \geq 0$, $t$ should be a positive real value such that $\left(\theta^{(1)}\right)^2 + \cdots + \left(\theta^{(m)}\right)^2 < 1$. Therefore,
\begin{align*}
    0 < t < \frac{1}{\sqrt{\left(a^{(1)}\right)^2 + \cdots + \left(a^{(m)}\right)^2}}
\end{align*}
Note that for any $t_s < t_l$, 
\begin{align*}
k(t_s) = \thetav(t_s)^T\av < \thetav(t_l)^T\av = k(t_l)
\end{align*}
where $\av = \left[ a^{(1)}, \cdots, a^{(m)} \right]^T$. Therefore, the plane has the maximum $k$ when it touches the hypersphere and the tangent point on its surface is the optimal solution of $\thetav$. The tangent point where the plane touches the hypersphere is at $t = \left(\sqrt{\left(a^{(1)}\right)^2 + \cdots + \left(a^{(m)}\right)^2}\right)^{-1}$, hence, the tangent point is:
\begin{align*}
    \thetav = \left( \frac{a^{(1)}}{\sqrt{\left(a^{(1)}\right)^2 + \cdots + \left(a^{(m)}\right)^2}}, \cdots, \frac{a^{(m)}}{\sqrt{\left(a^{(1)}\right)^2 + \cdots + \left(a^{(m)}\right)^2}} \right)\\
\end{align*}
which will be the optimal solution of the optimization problem \ref{theta-closed-sol}.
\end{proof}
\bigskip
    ~\newpage
\begin{proposition}  \label{lem:trick2}
	If $\Kv^{(v)}$ is a $n \times n$ positive semidefinite matrix, $\textbf{tr}\bb{\Kv^{(v)}-\Hv^T\Kv^{(v)}\Hv}$ is non-negative for any real valued $n \times k$ matrix $\Hv$ such that $\Hv^T\Hv = \Iv_k$
\end{proposition}
\begin{proof}~\\~\\
	From Theorem \ref{theo:kyfan}, 
	\begin{align*}
		\min_{\Hv \in \Rc^{n\times k}, \Hv^T\Hv = \Iv_k} \textbf{tr}\bb{\Kv^{(v)}-\Hv^T\Kv^{(v)}\Hv} & = \textbf{tr}\bb{\Kv^{(v)}}-\max_{\Hv \in \Rc^{n\times k}, \Hv^T\Hv = \Iv_k}\textbf{tr}\bb{\Hv^T\Kv^{(v)}\Hv}\\
		& = \sum_{i=1}^{n} \lambda_i - \sum_{i = 1}^{k} \lambda_i \\
		& = \sum_{i=k+1}^{n} \lambda_i
	\end{align*} 
	Since $\Kv^{(v)}$ is positive semidefinite, all its eigenvalues are non-negative. Therefore, 
	\begin{align*}
	\min_{\Hv \in \Rc^{n\times k}, \Hv^T\Hv = \Iv_k} \textbf{tr}\bb{\Kv^{(v)}-\Hv^T\Kv^{(v)}\Hv} = \sum_{i=k+1}^{n} \lambda_i \geq 0
	\end{align*} 	
	Finally, for any real valued $n \times k$ matrix $\Hv$ such that $\Hv^T\Hv = \Iv_k$, 
	\begin{align*}
		\textbf{tr}\bb{\Kv^{(v)}-\Hv^T\Kv^{(v)}\Hv} \geq \min_{\Hv \in \Rc^{n\times k}, \Hv^T\Hv = \Iv_k} \textbf{tr}\bb{\Kv^{(v)}-\Hv^T\Kv^{(v)}\Hv} \geq 0
	\end{align*}
\end{proof}
\bigskip
    ~\newpage
    \begin{subtext}\label{centering_scaling}
\textbf{Centering and scaling.} 
\end{subtext}
    At every iteration, we must center the combined map $\phiv_{\thetav}\bb{\xv_i}$ around the origin before we perform (kernel) PCA and update the cluster assignments $\Hv$. That is, at every iteration, we must center the data by using the following kernel trick: $\Kv_{\thetav} \leftarrow \Kv_{\thetav}\text{ -- }\Jv_n \Kv_{\thetav}\text{ -- }\Kv_{\thetav} \Jv_n + \Jv_n \Kv_{\thetav} \Jv_n$ where $\Jv_n = \mathbf{1}_n\mathbf{1}_n^T/n$ \citep{scholkopf1998nonlinear}. 
    This is computationally inefficient. Therefore, we suggest the following proposition. 
    Using this proposition, we center $\Kv^{(v)}$ for each view only at the beginning of the algorithm instead of centering the combined kernel matrix $\Kv_{\thetav}$ at every iterations. 
    \begin{proposition} \label{proposition_centering_proof}
    Let $\widetilde{\Kv}_{\thetav}^* = \sum_{v=1}^{m}\theta^{(v)}\widetilde{\Kv}^{(v)}$ where $\widetilde{\Kv}^{(v)} = \Kv^{(v)} - \Jv_n\Kv^{(v)} - \Kv^{(v)}\Jv_n + \Jv_n\Kv^{(v)}\Jv_n$ for $v = 1, \cdots, m$. Then $\widetilde{\Kv}_{\thetav}^* = \widetilde{\Kv}_{\thetav}$ where $\widetilde{\Kv}_{\thetav} = \Kv_{\thetav} - \Jv_n \Kv_{\thetav} - \Kv_{\thetav} \Jv_n + \Jv_n \Kv_{\thetav} \Jv_n$ for any $\thetav \in \Rc^{m}$.
    \end{proposition}
    \begin{proof}
    	\begin{align*}
    	\widetilde{\Kv_{\thetav}}^* & = \sum_{v=1}^{m} \theta^{(v)} \widetilde{\Kv}^{(v)}\\
    	& = \sum_{v=1}^{m} \theta^{(v)} \bb{\Kv^{(v)} - \Jv_n \Kv^{(v)} - \Kv^{(v)} \Jv_n + \Jv_n \Kv^{(v)} \Jv_n}\\
    	& = \sum_{v=1}^{m} \theta^{(v)} \Kv^{(v)} - \sum_{v=1}^{m} \theta^{(v)} \Jv_n \Kv^{(v)} - \sum_{v=1}^{m} \theta^{(v)} \Kv^{(v)} \Jv_n + \sum_{v=1}^{m} \theta^{(v)} \Jv_n \Kv^{(v)} \Jv_n\\
    	& = \Kv_{\thetav} - \Jv_n \bb{\sum_{v=1}^{m} \theta^{(v)} \Kv^{(v)}} - \bb{\sum_{v=1}^{m} \theta^{(v)} \Kv^{(v)}} \Jv_n + \Jv_n \bb{\sum_{v=1}^{m} \theta^{(v)} \Kv^{(v)}} \Jv_n\\
    	& = \Kv_{\thetav} - \Jv_n \Kv_{\thetav} - \Kv_{\thetav} \Jv_n + \Jv_n \Kv_{\thetav} \Jv_n\\
    	& = \widetilde{\Kv_{\thetav}}
    	\end{align*}
    \end{proof}
    It is known that estimation of kernel coefficients depends on how the kernel matrices are scaled \citep{kloft2011lp,ong2008automated}. In order to make multiple views comparable to each other, we suggest to scale each kernel matrix before combining them by ${\Kv}^{(v)} \leftarrow {\Kv^{(v)}}/\mathbf{tr} \bb{\Kv^{(v)}}$. Note that the trace of the centered kernel matrix is the sum of its eigenvalues, i.e. $\mathbf{tr} \bb{\Kv^{(v)}} = \sum_{i=1}^{n} \lambda_i^{(v)}$, which can be interpreted as the measure of variance explained by principal components of the feature space $\Fc$ within each view. Therefore, by scaling the kernel matrix, 
    the total variance explained within each view is set to be uniform, i.e. $\mathbf{tr}\bb{\Kv^{(1)}} = \cdots = \mathbf{tr}\bb{\Kv^{(m)}} = 1$.
\bigskip

    ~\newpage
    \begin{subtext}\label{sim-detail}
\textbf{Simulation Detail.}
\end{subtext}
We evaluate robustness of our method against two types of adversarial perturbations:
	\begin{itemize}
		\item Noise variables that are independently sampled from Gaussian distribution with zero-mean and unit-variance. We add different numbers ($N_{noise} = 0, 1, 2, \cdots$) of noise variables to a view.
		\item Redundant variables that are correlated with original variables. We add different numbers ($N_{redun} = 1, 2, \cdots$) of variables having different correlations ($cor = 1, 0.97, 0.90, 0.72, 0.45$) with the original variables to a view.
	\end{itemize}
Under these perturbations, we examine how our method make use of complementary patterns in multiple views. For this purpose, we first generated multiview data in three scenarios A--C with two or three views. Those views have complementary patterns necessary for identifying true clusters. Scenario A is composed of a complete view that has complete information to detect the three clusters and a partial view that only conveys partial information. Scenario B is composed of two different partial views so that each view alone cannot completely detect the three clusters. Both scenarios A \& B aim to test how the compared methods use the complementary information in two views. Scenario C is composed of two different partial views and a noise view. It aims to test further whether the methods robustly use complementary information from views even when one of the views contains only noise variables. Then, we added different types and levels of adversarial perturbations to one of the views. We denote the simulation data with the noise variables by A-Noise, B-Noise, and C-Noise, and the data with the redundant variables by A-Redun, B-Redun, and C-Redun.

All features were standardized so that they are centered around zero with standard deviations of one. A kernel function $\kv\bb{\xv, \yv} = \exp \bb{\text{-- }0.5||\xv\text{ -- }\yv||^2}$ was used for all the views. After obtaining continuous clustering indicator $\Hv^*$, we performed $k$-means clustering on the normalized $\Hv^*$ with 1000 random starts and reported the best result minimizing the objective function. We stopped the iteration if the stopping criteria $||\thetav_t\text{ -- }\thetav_{t\text{--}1}||_2 < 10^{\text{ --}4}$ is met within 500 iterations.

We compared MML-MKKC with seven other methods: two baseline methods, four recently proposed MKKC methods, and one variant of MML-MKKC. In particular, we included the following baseline methods: 
\begin{itemize}
	\item \textbf{Single Best} uses the best view that minimizes the kernel $k$-means objective function (\ref{exp:opt03}).
	\item \textbf{Uniform Weight} equally assigns all the kernel coefficients $\thetav$ to all views. It takes the combined kernel $\Kv_{\thetav}=\sum_{v=1}^{m}\Kv^{(v)}/m$ as an input $\Kv$ in the problem (\ref{exp:opt03}).
\end{itemize}
The following three MKKC methods are similar in that they all combine multiple kernels as: $\Kv_{\thetav} = \sum_{v=1}^{m}{\theta^{(v)}}^2\Kv^{(v)}$, with $l_1$ constraint on the kernel coefficients $\thetav$, and solve the problem (\ref{exp:opt21}) using the $\min_{\Hv}$-$\min_{\thetav}$ framework.
\begin{itemize}
	\item \textbf{Gonen's MKK} \citep{gonen2014localized} 
	\item \textbf{Gonen's LMKK} \citep{gonen2014localized}: This localized multiple kernel $k$-means clustering method aims to capture sample-specific characteristic of multiple data sources by estimating sample-specific kernel coefficients.	
	\item \textbf{Liu's MKK-MIR} \citep{liu2016multiple}: This method characterizes the correlation of each pair of kernels by integrating a matrix-induced quadratic regularization into the objective function. The regularization parameter $\lambda$ was set to 1 and the quadratic coefficient matrix $\Mv$ was defined as suggested by the paper. 
\end{itemize}
The fourth MKKC method combines the multiple kernels in a different way: 
\begin{itemize}
	\item \textbf{Yu's OKKC} \citep{yu2012optimized}: This method combines multiple views as $\Kv_{\thetav} = \sum_{v=1}^{m}\theta^{(v)}\Kv^{(v)}$ and uses $l_p$ constraint on $\thetav$ where $p \geq 1$, and optimize the problem (\ref{exp:opt21}) using the $\max_{\Hv}$-$\max_{\thetav}$ framework. However, rather than minimizing $\textbf{tr} \bb{\Kv_{\thetav}}\text{-- }\textbf{tr} \bb{\Hv^T\Kv_{\thetav}\Hv}$ as the general formula (\ref{exp:opt21}) does,
	it maximizes the objective function $\textbf{tr} \bb{\Hv^T\Kv_{\thetav}\Hv}$ so that it also leads to solutions favor assigning more weights to dominant views.
	The original algorithm iteratively optimizes the kernel coefficients $\thetav$ and discrete clustering assignment, which increases computational burden and costs more time. For a fair comparison, we updated the continuous cluster assignment $\Hv$ instead of retaining the discrete assignment at every iteration, and optimized it as QCLP, as proposed by all the other MKKC methods including ours.
\end{itemize}
Finally, we also include a variant of our method in the comparison:
\begin{itemize}
	\item \textbf{MinMax-MinC} is the $l_1$-regularization version of our method MML-MKKC, which is included to examine the effect of $l_2$-regularization in our method on clustering. It uses the same $\min_{\Hv}$-$\max_{\thetav}$ formulation in the problem (\ref{exp:opt13}) as our method but with $l_1$ instead of $l_2$ constraint on $\thetav$. Additionally, it uses $\thetav \geq \thetav_{min}$ where $\thetav_{min} = 0.5/m \mathbf{1}$ to avoid a sparse trivial solution.
\end{itemize}
    ~\newpage
    \begin{subtext}\label{real-detail}
\textbf{Data preprocessing and strategy.} 
\end{subtext}
The mRNA and methylation data are log-transformed. Variables have more than 5\% missing values are excluded, otherwise imputed using KNNimpute \citep{hastie1999imputing}. For each cancer, the top 100 features with largest median absolute deviation across the samples are used for each view. A radial basis function kernel is used for all views as suggested by 
To avoid the kernel matrices getting zero values due to a large number of features, we set the parameter of the radial basis function kernel as $\sigma = 1/(2 + p^2)$ where $p$ is the number of features. The kernel matrices were centered and scaled as described in Section \ref{centering_scaling}. 
%

    ~\newpage
    ~~	
    \begin{figure}[t]
		\begin{center}
			\includegraphics[width=0.5\linewidth]{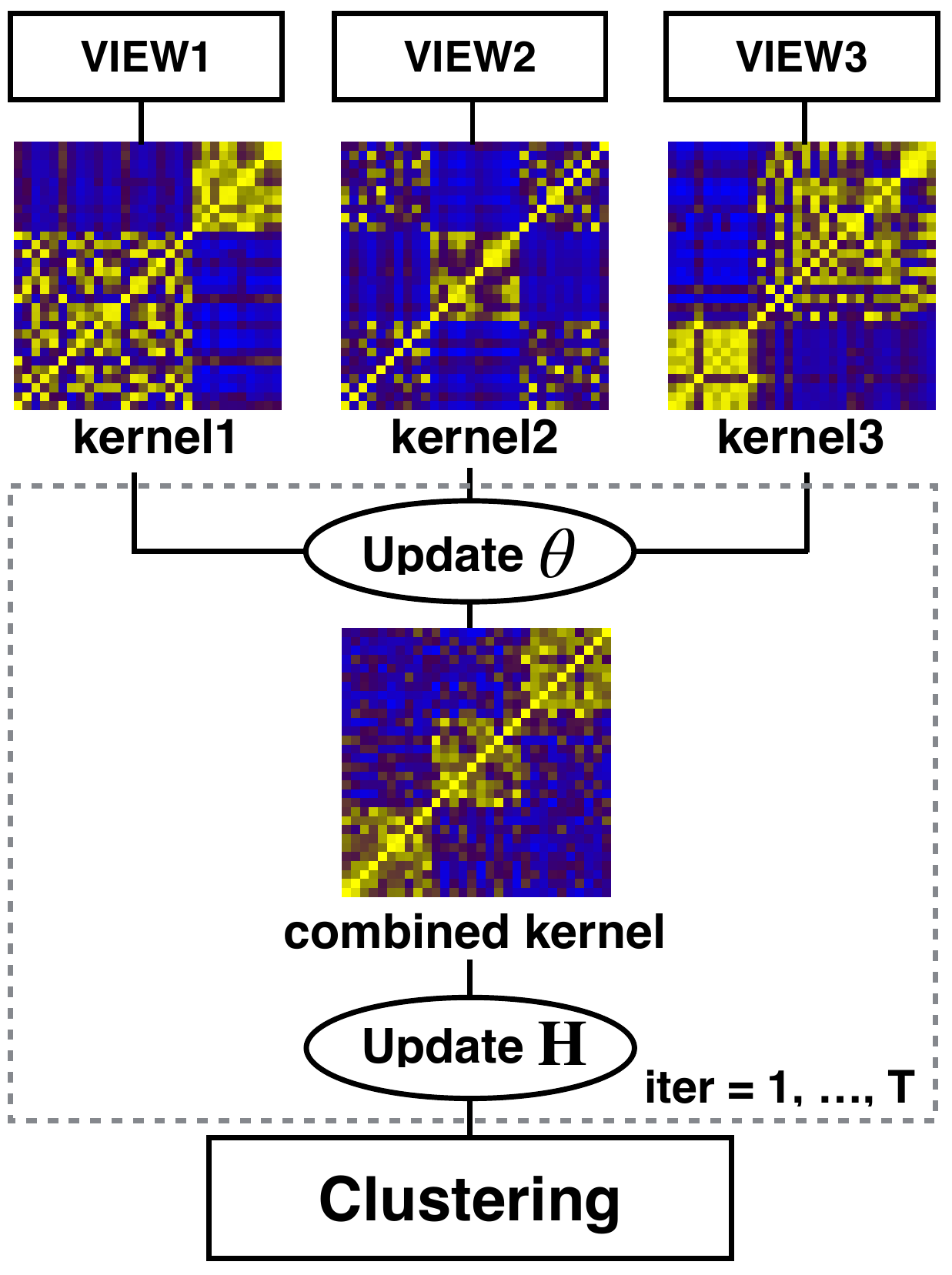}
		\end{center}
		\caption{Overview of multiple kernel clustering. It combines multiple views by taking a linear sum of multiple kernels where each kernel captures similarity between samples within each view. The kernel coefficients $\thetav$ and the cluster assignment matrix $\Hv$ are alternately optimized given each other.
		}\label{fig:overview}
	\end{figure}
    ~\newpage
    ~~
\begin{figure}[!ht]
	\centering
	\includegraphics[width=\textwidth]{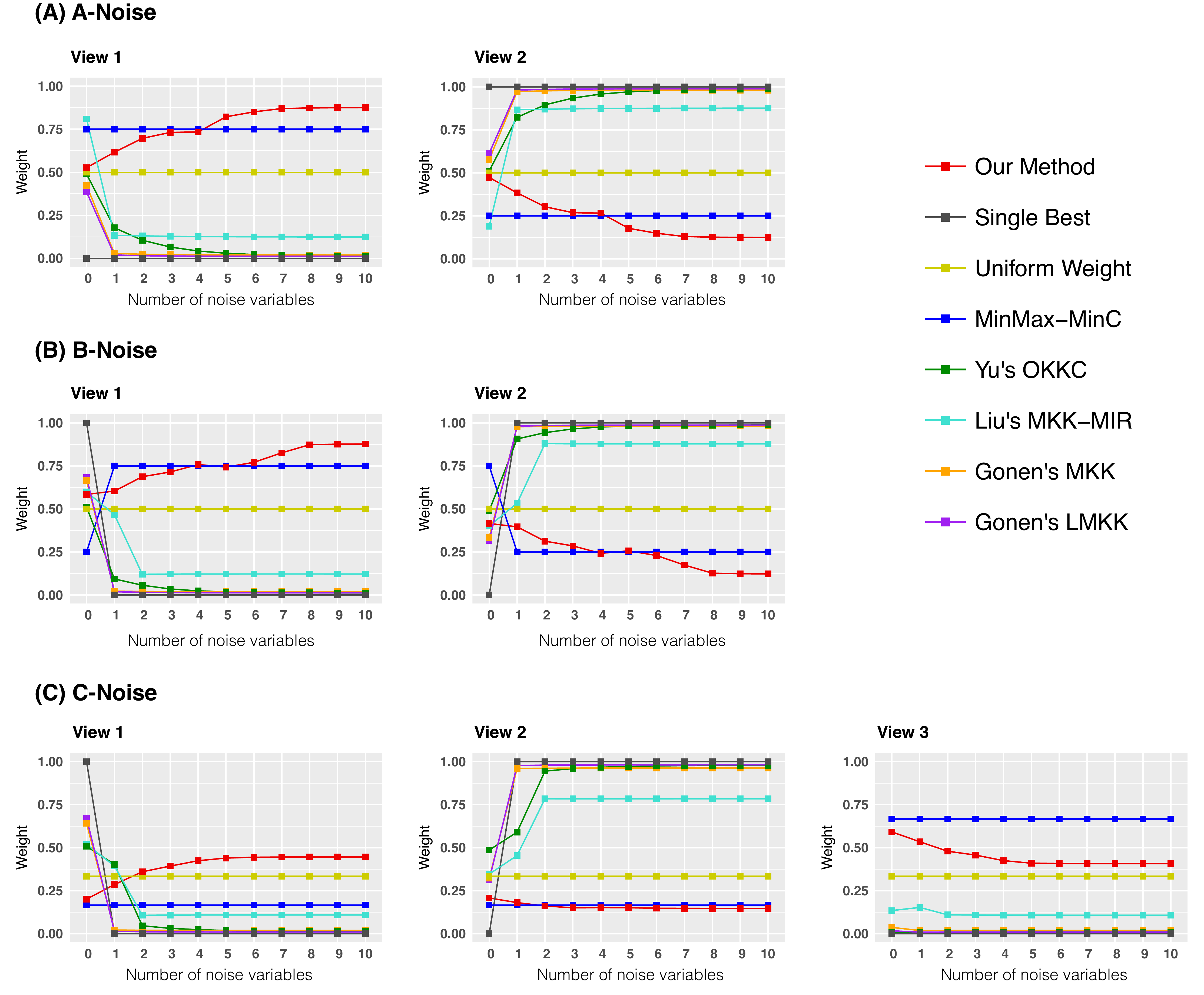}
	\caption{\textbf{Weights given by the compared methods to the views when Scenarios A-Noise, B-Noise, and C-Noise were used to identify clusters.} The weights on the views are plotted against the number of the noise variables $\Nv_{noise}$ added to the first view in each scenario. The x-axis represents the number of noise variables added to the first view. The y-axis represents the relative weight given by the compared methods. The methods are identified by different colors. For comparison purposes, we defined the weight as $\thetav / \thetav^T\mathbf{1}$ for the methods combining kernels using $\Kv_{\thetav} = \sum_{v=1}^{m}\theta^{(v)}\Kv^{(v)}$ (such as Uniform Weight, MinMax-MinC, Yu's OKKC, and our method) and as $\thetav^2 / {\thetav^2}^T\mathbf{1}$ for the methods using $\Kv_{\thetav} = \sum_{v=1}^{m}{\theta^{(v)}}^2\Kv^{(v)}$ (such as Liu's MIR, Gonen's MKK and LMKK).}\label{fig:thetanoise}
\end{figure}
\bigskip ~~

    ~
    \newpage
    
\begin{figure}[!ht]
	\centering
	\includegraphics[width=\textwidth]{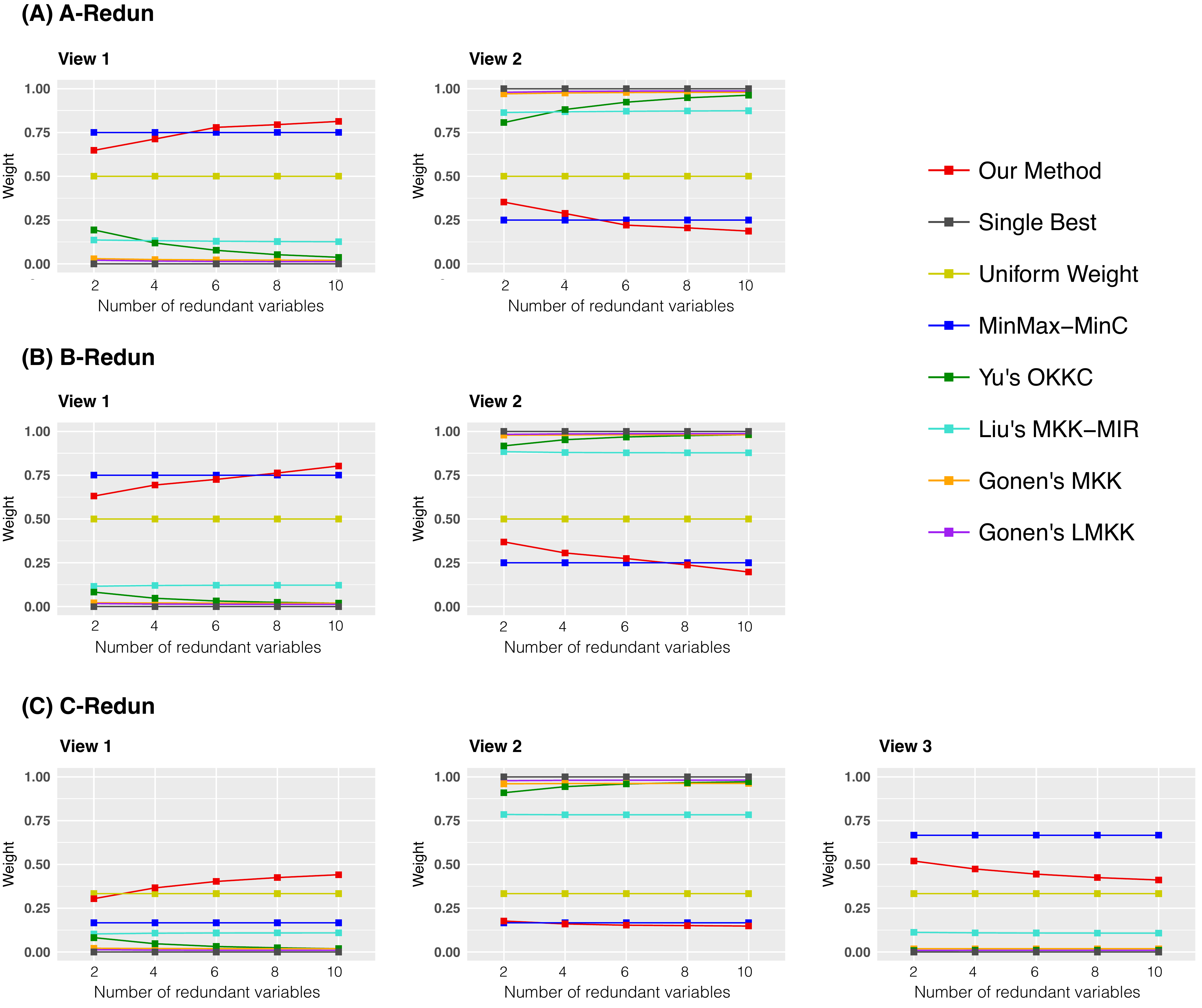}
	\caption{\textbf{Weights given by the compared methods to the views when Scenario A-2, B-2, and C-2 were used to identify clusters.} The weights on the views are plotted against the redundant variables $\Nv_{redun}$ where $cor = 0.90$ added to the first view in each scenario. The x-axis represents the number of noise variables added to the first view. The y-axis represents the relative weight given by the compared methods. The methods are identified by different colors. For comparison purposes, we defined the weight as $\thetav / \thetav^T\mathbf{1}$ for the methods combining kernels using $\Kv_{\thetav} = \sum_{v=1}^{m}\theta^{(v)}\Kv^{(v)}$ (such as Uniform Weight, MinMax-MinC, Yu's OKKC, and our method) and as $\thetav^2 / {\thetav^2}^T\mathbf{1}$ for the methods using $\Kv_{\thetav} = \sum_{v=1}^{m}{\theta^{(v)}}^2\Kv^{(v)}$ (such as Liu's MIR, Gonen's MKK and LMKK).}\label{fig:thetaredun}
\end{figure}
\bigskip ~~~

    ~
    \newpage
\begin{table}[ht]
	\centering
	\caption[caption]{\textbf{Evaluation of the Clutering Methods on Scenario A-Noise}} \label{table:evalclusteringAnoise}
	\begin{adjustbox}{max width=\textwidth}
		\def\arraystretch{0.95}
		\begin{threeparttable}
		\begin{tabular}{@{}p{2.5cm}lccccccccccc@{}}
			\toprule
			\multicolumn{13}{l}{\large \textbf{Scenario A-1}} \cr \midrule
			& $\Nv_{noise}$  & \textbf{0}     & \textbf{1}     & \textbf{2}     & \textbf{3}     & \textbf{4}     
			& \textbf{5}     & \textbf{6}     & \textbf{7}     & \textbf{8}     & \textbf{9}     & \textbf{10}  \cr \midrule 
			\multirow{3}{*}{\begin{tabular}{@{}l@{}}Single\cr Best\end{tabular}}       & adjRI  & 0.497 & 0.497 & 0.497 & 0.497 & 0.497 & 0.497 & 0.497 & 0.497 & 0.497 & 0.497 & 0.497\rule{0pt}{3.3ex} \cr
			& normMI & 0.837 & 0.837 & 0.837 & 0.837 & 0.837 & 0.837 & 0.837 & 0.837 & 0.837 & 0.837 & {\bf 0.837}\cr
			& purity & 0.680 & 0.680 & 0.680 & 0.680 & 0.680 & 0.680 & 0.680 & 0.680 & 0.680 & 0.680 & 0.680\vspace{1.5ex} \cr \midrule 
			\multirow{3}{*}{\begin{tabular}{@{}l@{}}Uniform\cr Weight\end{tabular}}       & adjRI  & \textbf{1.000} & \textbf{1.000} & 0.795 & 0.522 & 0.502 & 0.500 & 0.499 & 0.499 & 0.498 & 0.497 & {\bf 0.498}\rule{0pt}{3.3ex}\cr
			& normMI & \textbf{1.443} & \textbf{1.443} & 1.137 & 0.858 & 0.840 & 0.839 & 0.838 & 0.838 & 0.837 & 0.837 & {\bf 0.837} \cr
			& purity & \textbf{1.000} & \textbf{1.000} & 0.923 & 0.740 & 0.700 & 0.693 & 0.687 & 0.687 & 0.680 & 0.680 & \textbf{0.687}\vspace{1.5ex} \cr \midrule
			\multirow{3}{*}{\begin{tabular}{@{}l@{}}MinMax\cr MinC\end{tabular}}        & adjRI  & \textbf{1.000} & \textbf{1.000} & 0.795 & 0.522 & 0.502 & 0.500 & 0.499 & 0.499 & 0.498 & 0.497 & \textbf{0.498}\rule{0pt}{3.3ex}\cr
			& normMI & \textbf{1.443} & \textbf{1.443} & 1.137 & 0.858 & 0.840 & 0.839 & 0.838 & 0.838 & 0.837 & 0.837 & \textbf{0.837}\cr
			& purity & \textbf{1.000} & \textbf{1.000} & 0.923 & 0.740 & 0.700 & 0.693 & 0.687 & 0.687 & 0.680 & 0.680 & \textbf{0.687}\vspace{1.5ex} \cr \midrule
			\multirow{3}{*}{\begin{tabular}{@{}l@{}}Yu's\cr OKKC\end{tabular}}  & adjRI  & \textbf{1.000} & 0.497 & 0.497 & 0.497 & 0.497 & 0.497 & 0.497 & 0.497 & 0.497 & 0.497 & 0.497\rule{0pt}{3.3ex}\cr
			& normMI & \textbf{1.443} & 0.837 & 0.837 & 0.837 & 0.837 & 0.837 & 0.837 & 0.837 & 0.837 & 0.837 & \textbf{0.837}\cr
			& purity & \textbf{1.000} & 0.680 & 0.680 & 0.680 & 0.680 & 0.680 & 0.680 & 0.680 & 0.680 & 0.680 & 0.680\vspace{1.5ex} \cr \midrule
			\multirow{3}{*}{\begin{tabular}{@{}l@{}}Liu's\cr MKK-MIR\end{tabular}}      & adjRI  & \textbf{1.000} & 0.500 & 0.500 & 0.498 & 0.498 & 0.497 & 0.497 & 0.497 & 0.497 & 0.497 & 0.497\rule{0pt}{3.3ex}\cr
			& normMI & \textbf{1.443} & 0.839 & 0.839 & 0.837 & 0.837 & 0.837 & 0.837 & 0.837 & 0.837 & 0.837 & \textbf{0.837} \cr
			& purity & \textbf{1.000} & 0.693 & 0.693 & 0.683 & 0.683 & 0.680 & 0.680 & 0.680 & 0.680 & 0.680 & 0.680\vspace{1.5ex} \cr \midrule
			\multirow{3}{*}{\begin{tabular}{@{}l@{}}Gonen's\cr MKK \end{tabular}}    & adjRI  & \textbf{1.000} & 0.497 & 0.497 & 0.497 & 0.497 & 0.497 & 0.497 & 0.497 & 0.497 & 0.497 & 0.497\rule{0pt}{3.3ex}\cr
			& normMI & \textbf{1.443} & 0.837 & 0.837 & 0.837 & 0.837 & 0.837 & 0.837 & 0.837 & 0.837 & 0.837 & \textbf{0.837} \cr
			& purity & \textbf{1.000} & 0.680 & 0.680 & 0.680 & 0.680 & 0.680 & 0.680 & 0.680 & 0.680 & 0.680 & 0.680\vspace{1.5ex} \cr \midrule
			\multirow{3}{*}{\begin{tabular}{@{}l@{}}Gonen's\cr LMKK \end{tabular}} & adjRI  & 0.980 & 0.498 & 0.498 & 0.498 & 0.498 & 0.498 & 0.498 & 0.498 & 0.498 & 0.498 & \textbf{0.498}\rule{0pt}{3.3ex}\cr
			& normMI & 1.400 & 0.837 & 0.837 & 0.837 & 0.837 & 0.837 & 0.837 & 0.837 & 0.837 & 0.837 & \textbf{0.837}\cr
			& purity & 0.993 & 0.683 & 0.683 & 0.683 & 0.683 & 0.683 & 0.683 & 0.683 & 0.683 & 0.683 & 0.683\vspace{1.5ex}\cr \midrule
			\multirow{3}{*}{Our Method}      & adjRI  & \textbf{1.000} & \textbf{1.000} & \textbf{1.000} & \textbf{0.951} & \textbf{0.666} & \textbf{0.649} & \textbf{0.548} & \textbf{0.508} & \textbf{0.503} & \textbf{0.501} & \textbf{0.498}\rule{0pt}{3.3ex}\cr
			& normMI & \textbf{1.443} & \textbf{1.443} & \textbf{1.443} & \textbf{1.355} & \textbf{0.993} & \textbf{0.976} & \textbf{0.882} & \textbf{0.846} & \textbf{0.841} & \textbf{0.840} & \textbf{0.837}\cr
			& purity & \textbf{1.000} & \textbf{1.000} & \textbf{1.000} & \textbf{0.983} & \textbf{0.860} & \textbf{0.850} & \textbf{0.773} & \textbf{0.717} & \textbf{0.703} & \textbf{0.697} & 0.683\vspace{1.5ex}\cr \bottomrule 
		\end{tabular}%
	    \begin{tablenotes}
	        \footnotesize
	      	\item Clustering performance of the methods are evaluated by three widely-used metrics: Adjusted Rand Index (adjRI), Normalized Mutual Information (normMI), and purity. A higher value of the metrics indicates better clustering performance. Each column represents a simulated data set where the corresponding number indicates the number of the noise variables ($\Nv_{noise}$) added to the complete view (View 1). The bolded numbers are the maximum value for each the evalution measure within a simulation data set.
	    \end{tablenotes}
		\end{threeparttable}
	\end{adjustbox}	
\end{table}

    ~
    \newpage
\begin{table}[ht]
	\centering
	\caption[caption]{\textbf{Evaluation of the Clutering Methods on Scenario B-Noise}} \label{table:evalclusteringBnoise}
	\begin{adjustbox}{max width=\textwidth}
		\def\arraystretch{0.95}
		\begin{threeparttable}
		\begin{tabular}{@{}p{2.5cm}lccccccccccc@{}}
			\toprule
			\multicolumn{13}{l}{\large \textbf{Scenario B-1}} \cr \midrule
			& $\Nv_{noise}$  & \textbf{0}     & \textbf{1}     & \textbf{2}     & \textbf{3}     & \textbf{4}     
			& \textbf{5}     & \textbf{6}     & \textbf{7}     & \textbf{8}     & \textbf{9}     & \textbf{10}  \cr \midrule 
			\multirow{3}{*}{\begin{tabular}{@{}l@{}}Single\cr Best\end{tabular}}
			& adjRI  & 0.497 & 0.487 & 0.487 & 0.487 & 0.487 & 0.487 & 0.487 & 0.487 & 0.487 & \textbf{0.487} & \textbf{0.487} \rule{0pt}{3.3ex} \cr
			& normMI & 0.837 & 0.811 & 0.811 & 0.811 & 0.811 & 0.811 & 0.811 & 0.811 & 0.811 & 0.811 & \textbf{0.811} \cr
			& purity & 0.667 & 0.663 & 0.663 & 0.663 & 0.663 & 0.663 & 0.663 & 0.663 & 0.663 & 0.663 & 0.663 \vspace{1.5ex} \cr \midrule 
			\multirow{3}{*}{\begin{tabular}{@{}l@{}}Uniform\cr Weight\end{tabular}}  
			& adjRI  & \textbf{1.000} & \textbf{1.000} & 0.980 & 0.487 & 0.487 & 0.487 & 0.487 & 0.487 & 0.487 & \textbf{0.487} & \textbf{0.487} \rule{0pt}{3.3ex}\cr
			& normMI & \textbf{1.443} & \textbf{1.443} & 1.400 & 0.811 & 0.811 & 0.811 & 0.811 & 0.811 & 0.811 & 0.811 & \textbf{0.811} \cr
			& purity & \textbf{1.000} & \textbf{1.000} & 0.993 & 0.670 & 0.663 & 0.663 & 0.663 & 0.667 & 0.663 & 0.663 & 0.663 \cr \midrule
			\multirow{3}{*}{\begin{tabular}{@{}l@{}}MinMax \cr MinC\end{tabular}} 
			& adjRI  & \textbf{1.000} & \textbf{1.000} & 0.980 & 0.487 & 0.487 & 0.487 & 0.487 & 0.487 & 0.487 & \textbf{0.487} & \textbf{0.487} \rule{0pt}{3.3ex}\cr
			& normMI & \textbf{1.443} & \textbf{1.443} & 1.400 & 0.811 & 0.811 & 0.811 & 0.811 & 0.811 & 0.811 & 0.811 & \textbf{0.811} \cr
			& purity & \textbf{1.000} & \textbf{1.000} & 0.993 & 0.670 & 0.663 & 0.663 & 0.663 & 0.667 & 0.663 & 0.663 & 0.663 \cr \midrule
			\multirow{3}{*}{\begin{tabular}{@{}l@{}}Yu's\cr OKKC\end{tabular}}
            & adjRI  & \textbf{1.000} & 0.487 & 0.487 & 0.487 & 0.487 & 0.487 & 0.487 & 0.487 & 0.487 & \textbf{0.487} & \textbf{0.487} \rule{0pt}{3.3ex}\cr
			& normMI & \textbf{1.443} & 0.811 & 0.811 & 0.811 & 0.811 & 0.811 & 0.811 & 0.811 & 0.811 & 0.811 & \textbf{0.811} \cr
			& purity & \textbf{1.000} & 0.663 & 0.663 & 0.663 & 0.663 & 0.663 & 0.663 & 0.663 & 0.663 & 0.663 & 0.663 \cr \midrule
			\multirow{3}{*}{\begin{tabular}{@{}l@{}}Liu's\cr MKK-MIR\end{tabular}}
			& adjRI  & \textbf{1.000} & \textbf{1.000} & 0.487 & 0.487 & 0.487 & 0.487 & 0.487 & 0.487 & 0.487 & \textbf{0.487} & \textbf{0.487} \rule{0pt}{3.3ex}\cr
			& normMI & \textbf{1.443} & \textbf{1.443} & 0.811 & 0.811 & 0.811 & 0.811 & 0.811 & 0.811 & 0.811 & 0.811 & \textbf{0.811} \cr
			& purity & \textbf{1.000} & \textbf{1.000} & 0.663 & 0.663 & 0.663 & 0.663 & 0.663 & 0.663 & 0.663 & 0.663 & 0.663 \cr \midrule
            \multirow{3}{*}{\begin{tabular}{@{}l@{}}Gonen's\cr MKK \end{tabular}}
			& adjRI  & \textbf{1.000} & 0.487 & 0.487 & 0.487 & 0.487 & 0.487 & 0.487 & 0.487 & 0.487 & \textbf{0.487} & \textbf{0.487} \rule{0pt}{3.3ex}\cr
			& normMI & \textbf{1.443} & 0.811 & 0.811 & 0.811 & 0.811 & 0.811 & 0.811 & 0.811 & 0.811 & 0.811 & \textbf{0.811} \cr
			& purity & \textbf{1.000} & 0.663 & 0.663 & 0.663 & 0.663 & 0.663 & 0.663 & 0.663 & 0.663 & 0.663 & 0.663 \cr \midrule
            \multirow{3}{*}{\begin{tabular}{@{}l@{}}Gonen's\cr LMKK \end{tabular}} 
			& adjRI  & 0.552 & 0.487 & 0.487 & 0.487 & 0.487 & 0.487 & 0.487 & 0.487 & 0.487 & \textbf{0.487} & \textbf{0.487} \rule{0pt}{3.3ex}\cr
			& normMI & 0.832 & 0.812 & 0.811 & 0.811 & 0.811 & 0.811 & 0.811 & 0.811 & 0.811 & 0.811 & \textbf{0.811} \cr
			& purity & 0.820 & 0.670 & 0.663 & 0.663 & 0.663 & 0.663 & 0.663 & 0.663 & 0.663 & 0.663 & 0.663 \cr \midrule
			\multirow{3}{*}{Our Method} 
			& adjRI  & \textbf{1.000} & \textbf{1.000} & \textbf{1.000} & \textbf{0.951} & \textbf{0.789} & \textbf{0.510} & \textbf{0.490} & \textbf{0.493} & \textbf{0.490} & \textbf{0.487} & \textbf{0.487} \rule{0pt}{3.3ex}\cr
			& normMI & \textbf{1.443} & \textbf{1.443} & \textbf{1.443} & \textbf{1.355} & \textbf{1.159} & \textbf{0.847} & \textbf{0.814} & \textbf{0.817} & \textbf{0.814} & \textbf{0.812} & \textbf{0.811} \cr
			& purity & \textbf{1.000} & \textbf{1.000} & \textbf{1.000} & \textbf{0.983} & \textbf{0.920} & \textbf{0.720} & \textbf{0.687} & \textbf{0.700} & \textbf{0.690} & \textbf{0.673} & \textbf{0.670} \vspace{1.5ex}\cr \bottomrule 
		\end{tabular}%
	    \begin{tablenotes}
	        \footnotesize
	      	\item Clustering performance of the methods are evaluated by three widely-used metrics: Adjusted Rand Index (adjRI), Normalized Mutual Information (normMI), and purity. A higher value of the metrics indicates better clustering performance. Each column represents a simulated data set where the corresponding number indicates the number of the noise variables ($\Nv_{noise}$) added to the first partial view (View 1). The bolded numbers are the maximum value for each the evalution measure within a simulation data set.
	    \end{tablenotes}
		\end{threeparttable}
	\end{adjustbox}	
\end{table}
    ~
    \newpage
\begin{table}[ht]
	\centering
	\caption[caption]{\textbf{Evaluation of the Clutering Methods on Scenario C-Noise}} \label{table:evalclusteringCnoise}
	\begin{adjustbox}{max width=\textwidth}
		\def\arraystretch{0.95}
		\begin{threeparttable}
		\begin{tabular}{@{}p{2.5cm}lccccccccccc@{}}
			\toprule
			\multicolumn{13}{l}{\large \textbf{Scenario C-1}} \cr \midrule
			& $\Nv_{noise}$  & \textbf{0}     & \textbf{1}     & \textbf{2}     & \textbf{3}     & \textbf{4}     
			& \textbf{5}     & \textbf{6}     & \textbf{7}     & \textbf{8}     & \textbf{9}     & \textbf{10}  \cr \midrule 
			\multirow{3}{*}{\begin{tabular}{@{}l@{}}Single\cr Best\end{tabular}}
			& adjRI  & 0.497 & 0.487 & 0.487 & 0.487 & 0.487 & 0.487 & 0.487 & 0.487 & 0.487 & 0.487 & 0.487 \rule{0pt}{3.3ex} \cr
			& normMI & 0.837 & 0.811 & 0.811 & 0.811 & 0.811 & 0.811 & 0.811 & 0.811 & 0.811 & 0.811 & 0.811 \cr
			& purity & 0.667 & 0.663 & 0.663 & 0.663 & 0.663 & 0.663 & 0.663 & 0.663 & 0.663 & 0.663 & 0.663 \vspace{1.5ex} \cr \midrule
			\multirow{3}{*}{\begin{tabular}{@{}l@{}}Uniform\cr Weight\end{tabular}}
			& adjRI  & \textbf{1.000} & \textbf{0.990} & 0.878 & 0.500 & 0.487 & 0.487 & 0.487 & 0.487 & 0.487 & 0.487 & 0.487 \rule{0pt}{3.3ex} \cr
			& normMI & \textbf{1.443} & \textbf{1.418} & 1.234 & 0.839 & 0.811 & 0.812 & 0.811 & 0.811 & 0.812 & 0.812 & 0.812 \cr
			& purity & \textbf{1.000} & \textbf{0.997} & 0.957 & 0.693 & 0.667 & 0.670 & 0.667 & 0.667 & 0.667 & 0.667 & 0.667 \vspace{1.5ex}\cr \midrule
			\multirow{3}{*}{\begin{tabular}{@{}l@{}}MinMax \cr MinC\end{tabular}} 
			& adjRI  & \textbf{1.000} & \textbf{0.990} & 0.878 & 0.500 & 0.487 & 0.487 & 0.487 & 0.487 & 0.487 & 0.487 & 0.487 \rule{0pt}{3.3ex} \cr
			& normMI & \textbf{1.443} & \textbf{1.418} & 1.234 & 0.839 & 0.811 & 0.812 & 0.811 & 0.811 & 0.812 & 0.812 & 0.812 \cr
			& purity & \textbf{1.000} & \textbf{0.997} & 0.957 & 0.693 & 0.667 & 0.670 & 0.667 & 0.667 & 0.667 & 0.667 & 0.667 \vspace{1.5ex}\cr \midrule
			\multirow{3}{*}{\begin{tabular}{@{}l@{}}Yu's\cr OKKC\end{tabular}} 
			& adjRI  & \textbf{1.000} & 0.545 & 0.487 & 0.487 & 0.487 & 0.487 & 0.487 & 0.487 & 0.487 & 0.487 & 0.487 \rule{0pt}{3.3ex} \cr
			& normMI & \textbf{1.443} & 0.879 & 0.811 & 0.811 & 0.811 & 0.811 & 0.811 & 0.811 & 0.811 & 0.811 & 0.811 \cr
			& purity & \textbf{1.000} & 0.770 & 0.663 & 0.663 & 0.663 & 0.663 & 0.663 & 0.663 & 0.663 & 0.663 & 0.663 \vspace{1.5ex}\cr \midrule
			\multirow{3}{*}{\begin{tabular}{@{}l@{}}Liu's\cr MKK-MIR\end{tabular}}
			& adjRI  & \textbf{1.000} & 0.980 & 0.487 & 0.487 & 0.487 & 0.487 & 0.487 & 0.487 & 0.487 & 0.487 & 0.487 \rule{0pt}{3.3ex} \cr
			& normMI & \textbf{1.443} & 1.394 & 0.811 & 0.811 & 0.811 & 0.811 & 0.811 & 0.811 & 0.811 & 0.811 & 0.811 \cr
			& purity & \textbf{1.000} & 0.993 & 0.663 & 0.663 & 0.663 & 0.663 & 0.663 & 0.663 & 0.663 & 0.663 & 0.663 \vspace{1.5ex}\cr \midrule
			\multirow{3}{*}{\begin{tabular}{@{}l@{}}Gonen's\cr MKK \end{tabular}}
			& adjRI  & \textbf{1.000} & 0.487 & 0.487 & 0.487 & 0.487 & 0.487 & 0.487 & 0.487 & 0.487 & 0.487 & 0.487 \rule{0pt}{3.3ex} \cr
			& normMI & \textbf{1.443} & 0.811 & 0.811 & 0.811 & 0.811 & 0.811 & 0.811 & 0.811 & 0.811 & 0.811 & 0.811 \cr
			& purity & \textbf{1.000} & 0.663 & 0.663 & 0.663 & 0.663 & 0.663 & 0.663 & 0.663 & 0.663 & 0.663 & 0.663 \vspace{1.5ex}\cr \midrule
			\multirow{3}{*}{\begin{tabular}{@{}l@{}}Gonen's\cr LMKK \end{tabular}}
			& adjRI  & 0.500 & 0.490 & 0.490 & 0.490 & 0.490 & 0.490 & 0.490 & 0.490 & 0.490 & 0.490 & 0.490 \rule{0pt}{3.3ex} \cr
			& normMI & 0.775 & 0.814 & 0.814 & 0.814 & 0.814 & 0.814 & 0.814 & 0.814 & 0.814 & 0.814 & 0.814 \cr
			& purity & 0.790 & 0.683 & 0.683 & 0.683 & 0.683 & 0.683 & \textbf{0.683} & \textbf{0.683} & 0.683 & \textbf{0.683} & \textbf{0.683} \vspace{1.5ex}\cr \midrule
			\multirow{3}{*}{Our Method}
			& adjRI  & \textbf{1.000} & \textbf{0.990} & \textbf{0.923} & \textbf{0.869} & \textbf{0.556} & \textbf{0.515} & \textbf{0.498} & \textbf{0.498} & \textbf{0.499} & \textbf{0.498} & \textbf{0.497} \rule{0pt}{3.3ex} \cr
			& normMI & \textbf{1.443} & \textbf{1.418} & \textbf{1.300} & \textbf{1.223} & \textbf{0.890} & \textbf{0.852} & \textbf{0.838} & \textbf{0.837} & \textbf{0.838} & \textbf{0.838} & \textbf{0.837} \cr
			& purity & \textbf{1.000} & \textbf{0.997} & \textbf{0.973} & \textbf{0.953} & \textbf{0.780} & \textbf{0.730} & 0.680 & 0.680 & \textbf{0.690} & \textbf{0.683} & 0.677 \vspace{1.5ex}\cr \bottomrule 
		\end{tabular}%
	    \begin{tablenotes}
	        \footnotesize
	      	\item Clustering performance of the methods are evaluated by three widely-used metrics: Adjusted Rand Index (adjRI), Normalized Mutual Information (normMI), and purity. A higher value of the metrics indicates better clustering performance. Each column represents a simulated data set where the corresponding number indicates the number of the noise variables ($\Nv_{noise}$) added to the first partial view (View 1). The bolded numbers are the maximum value for each the evalution measure within a simulation data set.
	    \end{tablenotes}
		\end{threeparttable}
	\end{adjustbox}	
\end{table}
    \newpage
\begin{sidewaystable}
	\centering
	\caption[caption]{\textbf{Evaluation of the Clutering Methods on Scenario A-Redun}} \label{table:evalclusteringAredun}
	\begin{adjustbox}{max width=\textwidth}
		\def\arraystretch{0.95}
		\begin{threeparttable}
			\begin{tabular}{@{}llccccccccccccccccccccccccc@{}}
				\toprule
				\multicolumn{27}{l}{\large \textbf{Scenario A-2}} \cr \midrule
				& {\bf cor}
				& \multicolumn{5}{c}{{\bf 0.45}}   & \multicolumn{5}{c}{\bf 0.72}  & \multicolumn{5}{c}{\bf 0.90}    
				& \multicolumn{5}{c}{\bf 0.97}   & \multicolumn{5}{c}{\bf 1}     \cr \cmidrule(lr){3-7} \cmidrule(lr){8-12} \cmidrule(lr){13-17} \cmidrule(lr){18-22} \cmidrule(lr){23-27} 
				& $\Nv_{redun}$     
				& \textbf{2}    & \textbf{4}    & \textbf{6}    & \textbf{8}    & \textbf{10}    & \textbf{2}    & \textbf{4}    & \textbf{6}    & \textbf{8}    & \textbf{10}    & \textbf{2}    & \textbf{4}    & \textbf{6}    & \textbf{8}    & \textbf{10}    & \textbf{2}    & \textbf{4}    & \textbf{6}    & \textbf{8}    & \textbf{10}    & \textbf{2}    & \textbf{4}    & \textbf{6}    & \textbf{8}    & \textbf{10}    \cr \midrule 
				\multirow{3}{*}{\begin{tabular}{@{}l@{}}Single\cr Best\end{tabular}}              
				& adjRI       & 0.497      & 0.497 & 0.497 & 0.497 & 0.497 & 0.497 & 0.497 & 0.497 & 0.497 & 0.497 & 0.497 & 0.497 & 0.497 & 0.497 & 0.497 & 0.497 & 0.497 & 0.497 & 0.497 & 0.497 & 0.497 & 0.497 & 0.497 & 0.497 & 0.497\rule{0pt}{3.3ex}\cr
				& normMI      & 0.837      & 0.837 & \textbf{0.837} & \textbf{0.837} & \textbf{0.837} & 0.837 & 0.837 & 0.837 & 0.837 & \textbf{0.837} & 0.837 & 0.837 & 0.837 & 0.837 & 0.837 & 0.837 & 0.837 & 0.837 & 0.837 & 0.837 & 0.837 & 0.837 & 0.837 & 0.837 & 0.837 \cr
				& purity       & 0.680      & 0.680 & 0.680 & 0.680 & 0.680 & 0.680 & 0.680 & 0.680 & 0.680 & 0.680 & 0.680 & 0.680 & 0.680 & 0.680 & 0.680 & 0.680 & 0.680 & 0.680 & 0.680 & 0.680 & 0.680 & 0.680 & 0.680 & 0.680 & 0.680\vspace{1.5ex} \cr \midrule
				\multirow{3}{*}{\begin{tabular}{@{}l@{}}Uniform\cr Weight\end{tabular}}              
				& adjRI        & \textbf{0.507}      & 0.497 & 0.497 & 0.497 & 0.497 & 0.795 & 0.499 & 0.497 & 0.497 & 0.497 & \textbf{0.980} & 0.951 & 0.515 & 0.501 & 0.499 & \textbf{1.000} & \textbf{1.000} & \textbf{1.000} & \textbf{1.000} & 0.904 & \textbf{1.000} & \textbf{1.000} & \textbf{1.000} & \textbf{1.000} & \textbf{0.990}\rule{0pt}{3.3ex}\cr
				& normMI       & \textbf{0.845}      & 0.837 & \textbf{0.837} & \textbf{0.837} & \textbf{0.837} & 1.132 & 0.839 & 0.836 & 0.837 & \textbf{0.837} & \textbf{1.400} & 1.341 & 0.852 & 0.840 & 0.838 & \textbf{1.443} & \textbf{1.443} & \textbf{1.443} & \textbf{1.443} & 1.270 & \textbf{1.443} & \textbf{1.443} & \textbf{1.443} & \textbf{1.443} & \textbf{1.418} \cr
				& purity       & \textbf{0.713}      & 0.680 & 0.680 & 0.680 & 0.680 & 0.923 & 0.683 & 0.677 & 0.680 & 0.680 & \textbf{0.993} & 0.983 & 0.730 & 0.697 & 0.687 & \textbf{1.000} & \textbf{1.000} & \textbf{1.000} & \textbf{1.000} & 0.967 & \textbf{1.000} & \textbf{1.000} & \textbf{1.000} & \textbf{1.000} & \textbf{0.997}\vspace{1.5ex} \cr \midrule
				\multirow{3}{*}{\begin{tabular}{@{}l@{}}MinMax \cr MinC\end{tabular}}              
				& adjRI        & \textbf{0.507}      & 0.497 & 0.497 & 0.497 & 0.497 & 0.795 & 0.499 & 0.497 & 0.497 & 0.497 & \textbf{0.980} & 0.951 & 0.515 & 0.501 & 0.499 & \textbf{1.000} & \textbf{1.000} & \textbf{1.000} & \textbf{1.000} & 0.904 & \textbf{1.000} & \textbf{1.000} & \textbf{1.000} & \textbf{1.000} & \textbf{0.990}\rule{0pt}{3.3ex}\cr
				& normMI       & \textbf{0.845}      & 0.837 & \textbf{0.837} & \textbf{0.837} & \textbf{0.837} & 1.132 & 0.839 & 0.836 & 0.837 & \textbf{0.837} & \textbf{1.400} & 1.341 & 0.852 & 0.840 & 0.838 & \textbf{1.443} & \textbf{1.443} & \textbf{1.443} & \textbf{1.443} & 1.270 & \textbf{1.443} & \textbf{1.443} & \textbf{1.443} & \textbf{1.443} & \textbf{1.418} \cr
				& purity       & \textbf{0.713}      & 0.680 & 0.680 & 0.680 & 0.680 & 0.923 & 0.683 & 0.677 & 0.680 & 0.680 & \textbf{0.993} & 0.983 & 0.730 & 0.697 & 0.687 & \textbf{1.000} & \textbf{1.000} & \textbf{1.000} & \textbf{1.000} & 0.967 & \textbf{1.000} & \textbf{1.000} & \textbf{1.000} & \textbf{1.000} & \textbf{0.997}\vspace{1.5ex} \cr \midrule
				\multirow{3}{*}{\begin{tabular}{@{}l@{}}Yu's\cr OKKC\end{tabular}}          
				& adjRI        & 0.497      & 0.497 & 0.497 & 0.497 & 0.497 & 0.497 & 0.497 & 0.497 & 0.497 & 0.497 & 0.497 & 0.497 & 0.497 & 0.497 & 0.497 & \textbf{1.000} & 0.497 & 0.497 & 0.497 & 0.497 & \textbf{1.000} & \textbf{1.000} & 0.519 & 0.497 & 0.497\rule{0pt}{3.3ex}\cr
				& normMI       & 0.837      & 0.837 & \textbf{0.837} & \textbf{0.837} & \textbf{0.837} & 0.837 & 0.837 & 0.837 & 0.837 & \textbf{0.837} & 0.837 & 0.837 & 0.837 & 0.837 & 0.837 & \textbf{1.443} & 0.837 & 0.837 & 0.837 & 0.837 & \textbf{1.443} & \textbf{1.443} & 0.856 & 0.837 & 0.837 \cr
				& purity       & 0.680      & 0.680 & 0.680 & 0.680 & 0.680 & 0.680 & 0.680 & 0.680 & 0.680 & 0.680 & 0.680 & 0.680 & 0.680 & 0.680 & 0.680 & \textbf{1.000} & 0.680 & 0.680 & 0.680 & 0.680 & \textbf{1.000} & \textbf{1.000} & 0.737 & 0.680 & 0.680\vspace{1.5ex} \cr \midrule
				\multirow{3}{*}{\begin{tabular}{@{}l@{}}Liu's\cr MKK-MIR\end{tabular}}             
				& adjRI        & 0.497      & 0.497 & 0.497 & 0.497 & 0.497 & 0.498 & 0.497 & 0.497 & 0.497 & 0.497 & 0.501 & 0.497 & 0.497 & 0.497 & 0.497 & \textbf{1.000} & 0.498 & 0.497 & 0.497 & 0.497 & \textbf{1.000} & \textbf{1.000} & 0.500 & 0.499 & 0.497\rule{0pt}{3.3ex}\cr
				& normMI       & 0.837      & 0.837 & \textbf{0.837} & \textbf{0.837} & \textbf{0.837} & 0.837 & 0.837 & 0.837 & 0.837 & \textbf{0.837} & 0.840 & 0.837 & 0.837 & 0.837 & 0.837 & \textbf{1.443} & 0.837 & 0.837 & 0.837 & 0.837 & \textbf{1.443} & \textbf{1.443} & 0.839 & 0.838 & 0.837 \cr
				& purity       & 0.680      & 0.680 & 0.680 & 0.680 & 0.680 & 0.680 & 0.680 & 0.680 & 0.680 & 0.680 & 0.697 & 0.680 & 0.680 & 0.680 & 0.680 & \textbf{1.000} & 0.687 & 0.680 & 0.680 & 0.680 & \textbf{1.000} & \textbf{1.000} & 0.693 & 0.690 & 0.680\vspace{1.5ex} \cr \midrule
				\multirow{3}{*}{\begin{tabular}{@{}l@{}}Gonen's\cr MKK \end{tabular}}           
				& adjRI        & 0.497      & 0.497 & 0.497 & 0.497 & 0.497 & 0.497 & 0.497 & 0.497 & 0.497 & 0.497 & 0.497 & 0.497 & 0.497 & 0.497 & 0.497 & 0.497 & 0.497 & 0.497 & 0.497 & 0.497 & 0.497 & 0.497 & 0.497 & 0.497 & 0.497\rule{0pt}{3.3ex}\cr
				& normMI       & 0.837      & 0.837 & \textbf{0.837} & \textbf{0.837} & \textbf{0.837} & 0.837 & 0.837 & 0.837 & 0.837 & \textbf{0.837} & 0.837 & 0.837 & 0.837 & 0.837 & 0.837 & 0.837 & 0.837 & 0.837 & 0.837 & 0.837 & 0.837 & 0.837 & 0.837 & 0.837 & 0.837\cr
				& purity       & 0.680      & 0.680 & 0.680 & 0.680 & 0.680 & 0.680 & 0.680 & 0.680 & 0.680 & 0.680 & 0.680 & 0.680 & 0.680 & 0.680 & 0.680 & 0.680 & 0.680 & 0.680 & 0.680 & 0.680 & 0.680 & 0.680 & 0.680 & 0.680 & 0.680\vspace{1.5ex} \cr \midrule
				\multirow{3}{*}{\begin{tabular}{@{}l@{}}Gonen's\cr LMKK \end{tabular}}        
				& adjRI        & 0.498      & 0.498 & \textbf{0.498} & \textbf{0.498} & \textbf{0.498} & 0.498 & 0.498 & \textbf{0.498} & 0.498 & \textbf{0.498} & 0.497 & 0.497 & 0.498 & 0.498 & 0.498 & 0.497 & 0.497 & 0.498 & 0.498 & 0.498 & 0.497 & 0.497 & 0.497 & 0.497 & 0.497\rule{0pt}{3.3ex}\cr
				& normMI       & 0.837      & 0.837 & \textbf{0.837} & \textbf{0.837} & \textbf{0.837} & 0.837 & 0.837 & 0.837 & 0.837 & \textbf{0.837} & 0.837 & 0.837 & 0.837 & 0.837 & 0.837 & 0.837 & 0.837 & 0.837 & 0.837 & 0.837 & 0.837 & 0.837 & 0.837 & 0.837 & 0.837 \cr
				& purity       & 0.683      & 0.683 & \textbf{0.683} & \textbf{0.683} & \textbf{0.683} & 0.683 & 0.683 & \textbf{0.683} & 0.683 & \textbf{0.683} & 0.680 & 0.680 & 0.683 & 0.683 & 0.683 & 0.680 & 0.680 & 0.683 & 0.683 & 0.683 & 0.680 & 0.680 & 0.680 & 0.680 & 0.680\vspace{1.5ex} \cr \midrule
				\multirow{3}{*}{Our Method}             
				& adjRI        & 0.497      & \textbf{0.502} & 0.497 & 0.497 & 0.497 & \textbf{0.941} & \textbf{0.788} & \textbf{0.498} & \textbf{0.499} & 0.497 & \textbf{0.980} & \textbf{0.990} & \textbf{1.000} & \textbf{1.000} & \textbf{0.844} & \textbf{1.000} & \textbf{1.000} & \textbf{1.000} & 0.980 & \textbf{0.970} & \textbf{1.000} & \textbf{1.000} & \textbf{1.000} & \textbf{1.000} & 0.970\rule{0pt}{3.3ex}\cr
				& normMI       & 0.836      & \textbf{0.841} & \textbf{0.837} & \textbf{0.837} & \textbf{0.837} & \textbf{1.325} & \textbf{1.127} & \textbf{0.838} & \textbf{0.838} & 0.836 & \textbf{1.400} & \textbf{1.418} & \textbf{1.443} & \textbf{1.443} & \textbf{1.203} & \textbf{1.443} & \textbf{1.443} & \textbf{1.443} & 1.400 & \textbf{1.384} & \textbf{1.443} & \textbf{1.443} & \textbf{1.443} & \textbf{1.443} & 1.384 \cr
				& purity       & 0.667      & \textbf{0.700} & 0.680 & 0.680 & 0.680 & \textbf{0.980} & \textbf{0.920} & \textbf{0.683} & \textbf{0.690} & 0.677 & \textbf{0.993} & \textbf{0.997} & \textbf{1.000} & \textbf{1.000} & \textbf{0.943} & \textbf{1.000} & \textbf{1.000} & \textbf{1.000} & 0.993 & \textbf{0.990} & \textbf{1.000} & \textbf{1.000} & \textbf{1.000} & \textbf{1.000} & 0.990\vspace{1.5ex} \cr \bottomrule
			\end{tabular}%
			\begin{tablenotes}
				\footnotesize
				\item Clustering performance of the methods are evaluated by three widely-used metrics: Adjusted Rand Index (adjRI), Normalized Mutual Information (normMI), and purity. A higher value of the metrics indicates better clustering performance. Each column represents a simulated data set where the corresponding number indicates the number of the redundant variables ($\Nv_{redun}$) added to the complete view (View 1) and correlation between each the redundant variables and the original variables. The bolded numbers are the maximum value for each the evalution measure within a simulation data set.
			\end{tablenotes}
		\end{threeparttable}
	\end{adjustbox}	
\end{sidewaystable}

    \newpage
\begin{sidewaystable}
	\centering
	\caption[caption]{\textbf{Evaluation of the Clutering Methods on Scenario B-Redun}} \label{table:evalclusteringBredun}
	\begin{adjustbox}{max width=\textwidth}
		\def\arraystretch{0.95}
		\begin{threeparttable}
			\begin{tabular}{@{}llccccccccccccccccccccccccc@{}}
				\toprule
				\multicolumn{27}{l}{\large \textbf{Scenario B-2}} \cr \midrule
				& {\bf cor}
				& \multicolumn{5}{c}{{\bf 0.45}}   & \multicolumn{5}{c}{\bf 0.72}  & \multicolumn{5}{c}{\bf 0.90}    
				& \multicolumn{5}{c}{\bf 0.97}   & \multicolumn{5}{c}{\bf 1}     \cr \cmidrule(lr){3-7} \cmidrule(lr){8-12} \cmidrule(lr){13-17} \cmidrule(lr){18-22} \cmidrule(lr){23-27} 
				& $\Nv_{redun}$     
				& \textbf{2}    & \textbf{4}    & \textbf{6}    & \textbf{8}    & \textbf{10}    & \textbf{2}    & \textbf{4}    & \textbf{6}    & \textbf{8}    & \textbf{10}    & \textbf{2}    & \textbf{4}    & \textbf{6}    & \textbf{8}    & \textbf{10}    & \textbf{2}    & \textbf{4}    & \textbf{6}    & \textbf{8}    & \textbf{10}    & \textbf{2}    & \textbf{4}    & \textbf{6}    & \textbf{8}    & \textbf{10}    \cr \midrule 
				\multirow{3}{*}{\begin{tabular}{@{}l@{}}Single\cr Best\end{tabular}}
				& adjRI   & 0.487      & \textbf{0.487} & \textbf{0.487} & \textbf{0.487} & \textbf{0.487} & 0.487 & 0.487 & \textbf{0.487} & \textbf{0.487} & \textbf{0.487} & 0.487 & 0.487 & 0.487 & 0.487 & 0.487 & 0.487 & 0.487 & 0.487 & 0.487 & 0.487 & 0.487 & 0.487 & 0.487 & 0.487 & 0.487 \rule{0pt}{3.3ex}\cr
				& normMI  & 0.811      & 0.811 & \textbf{0.811} & \textbf{0.811} & \textbf{0.811} & 0.811 & 0.811 & \textbf{0.811} & 0.811 & 0.811 & 0.811 & 0.811 & 0.811 & 0.811 & 0.811 & 0.811 & 0.811 & 0.811 & 0.811 & 0.811 & 0.811 & 0.811 & 0.811 & 0.811 & 0.811 \cr
				& purity  & 0.663      & 0.663 & 0.663 & \textbf{0.663} & \textbf{0.663} & 0.663 & 0.663 & 0.663 & 0.663 & 0.663 & 0.663 & 0.663 & 0.663 & 0.663 & 0.663 & 0.663 & 0.663 & 0.663 & 0.663 & 0.663 & 0.663 & 0.663 & 0.663 & 0.663 & 0.663 \vspace{1.5ex} \cr \midrule
				\multirow{3}{*}{\begin{tabular}{@{}l@{}}Uniform\cr Weight\end{tabular}}
				& adjRI   & 0.487      & \textbf{0.487} & \textbf{0.487} & \textbf{0.487} & \textbf{0.487} & 0.914 & 0.487 & \textbf{0.487} & \textbf{0.487} & \textbf{0.487} & \textbf{0.990} & 0.951 & 0.497 & 0.487 & 0.487 & \textbf{1.000} & \textbf{1.000} & \textbf{1.000} & \textbf{0.990} & 0.961 & \textbf{1.000} & \textbf{1.000} & \textbf{1.000} & \textbf{1.000} & \textbf{1.000} \rule{0pt}{3.3ex}\cr
				& normMI  & 0.812      & \textbf{0.812} & \textbf{0.811} & \textbf{0.811} & \textbf{0.811} & 1.288 & 0.811 & \textbf{0.811} & 0.811 & 0.811 & \textbf{1.418} & 1.355 & 0.837 & 0.811 & 0.811 & \textbf{1.443} & \textbf{1.443} & \textbf{1.443} & \textbf{1.418} & 1.359 & \textbf{1.443} & \textbf{1.443} & \textbf{1.443} & \textbf{1.443} & \textbf{1.443} \cr
				& purity  & 0.670      & 0.670 & 0.663 & \textbf{0.663} & \textbf{0.663} & 0.970 & 0.670 & \textbf{0.667} & 0.663 & 0.663 & \textbf{0.997} & 0.983 & 0.680 & 0.667 & 0.667 & \textbf{1.000} & \textbf{1.000} & \textbf{1.000} & \textbf{0.997} & 0.987 & \textbf{1.000} & \textbf{1.000} & \textbf{1.000} & \textbf{1.000} & \textbf{1.000} \vspace{1.5ex} \cr \midrule
				\multirow{3}{*}{\begin{tabular}{@{}l@{}}MinMax \cr MinC\end{tabular}}
				& adjRI   & 0.487      & \textbf{0.487} & \textbf{0.487} & \textbf{0.487} & \textbf{0.487} & 0.914 & 0.487 & \textbf{0.487} & \textbf{0.487} & \textbf{0.487} & \textbf{0.990} & 0.951 & 0.497 & 0.487 & 0.487 & \textbf{1.000} & \textbf{1.000} & \textbf{1.000} & \textbf{0.990} & 0.961 & \textbf{1.000} & \textbf{1.000} & \textbf{1.000} & \textbf{1.000} & \textbf{1.000} \rule{0pt}{3.3ex}\cr
				& normMI  & 0.812      & \textbf{0.812} & \textbf{0.811} & \textbf{0.811} & \textbf{0.811} & 1.288 & 0.811 & \textbf{0.811} & 0.811 & 0.811 & \textbf{1.418} & 1.355 & 0.837 & 0.811 & 0.811 & \textbf{1.443} & \textbf{1.443} & \textbf{1.443} & \textbf{1.418} & 1.359 & \textbf{1.443} & \textbf{1.443} & \textbf{1.443} & \textbf{1.443} & \textbf{1.443} \cr
				& purity  & 0.670      & 0.670 & 0.663 & \textbf{0.663} & \textbf{0.663} & 0.970 & 0.670 & \textbf{0.667} & 0.663 & 0.663 & \textbf{0.997} & 0.983 & 0.680 & 0.667 & 0.667 & \textbf{1.000} & \textbf{1.000} & \textbf{1.000} & \textbf{0.997} & 0.987 & \textbf{1.000} & \textbf{1.000} & \textbf{1.000} & \textbf{1.000} & \textbf{1.000} \vspace{1.5ex} \cr \midrule
				\multirow{3}{*}{\begin{tabular}{@{}l@{}}Yu's\cr OKKC\end{tabular}}
				& adjRI   & 0.487      & \textbf{0.487} & \textbf{0.487} & \textbf{0.487} & \textbf{0.487} & 0.487 & 0.487 & \textbf{0.487} & \textbf{0.487} & \textbf{0.487} & 0.487 & 0.487 & 0.487 & 0.487 & 0.487 & 0.869 & 0.487 & 0.487 & 0.487 & 0.487 & \textbf{1.000} & 0.519 & 0.487 & 0.487 & 0.487 \rule{0pt}{3.3ex}\cr
				& normMI  & 0.811      & 0.811 & \textbf{0.811} & \textbf{0.811} & \textbf{0.811} & 0.811 & 0.811 & \textbf{0.811} & 0.811 & 0.811 & 0.811 & 0.811 & 0.811 & 0.811 & 0.811 & 1.221 & 0.811 & 0.811 & 0.811 & 0.811 & \textbf{1.443} & 0.855 & 0.811 & 0.811 & 0.811 \cr
				& purity  & 0.663      & 0.663 & 0.663 & \textbf{0.663} & \textbf{0.663} & 0.663 & 0.663 & 0.663 & 0.663 & 0.663 & 0.663 & 0.663 & 0.663 & 0.663 & 0.663 & 0.953 & 0.663 & 0.663 & 0.663 & 0.663 & \textbf{1.000} & 0.737 & 0.663 & 0.663 & 0.663 \vspace{1.5ex} \cr \midrule
				\multirow{3}{*}{\begin{tabular}{@{}l@{}}Liu's\cr MKK-MIR\end{tabular}}
				& adjRI   & 0.487      & \textbf{0.487} & \textbf{0.487} & \textbf{0.487} & \textbf{0.487} & 0.487 & 0.487 & \textbf{0.487} & \textbf{0.487} & \textbf{0.487} & 0.487 & 0.487 & 0.487 & 0.487 & 0.487 & \textbf{1.000} & 0.487 & 0.487 & 0.487 & 0.487 & \textbf{1.000} & \textbf{1.000} & \textbf{1.000} & 0.487 & 0.487 \rule{0pt}{3.3ex}\cr
				& normMI  & 0.811      & 0.811 & \textbf{0.811} & \textbf{0.811} & \textbf{0.811} & 0.811 & 0.811 & \textbf{0.811} & 0.811 & 0.811 & 0.812 & 0.811 & 0.811 & 0.811 & 0.811 & \textbf{1.443} & 0.812 & 0.811 & 0.811 & 0.811 & \textbf{1.443} & \textbf{1.443} & \textbf{1.443} & 0.812 & 0.812 \cr
				& purity  & 0.663      & 0.663 & 0.663 & \textbf{0.663} & \textbf{0.663} & 0.663 & 0.663 & 0.663 & 0.663 & 0.663 & 0.670 & 0.663 & 0.663 & 0.663 & 0.663 & \textbf{1.000} & 0.670 & 0.663 & 0.663 & 0.663 & \textbf{1.000} & \textbf{1.000} & \textbf{1.000} & 0.670 & 0.670 \vspace{1.5ex} \cr \midrule
				\multirow{3}{*}{\begin{tabular}{@{}l@{}}Gonen's\cr MKK \end{tabular}}  
				& adjRI   & 0.487      & \textbf{0.487} & \textbf{0.487} & \textbf{0.487} & \textbf{0.487} & 0.487 & 0.487 & \textbf{0.487} & \textbf{0.487} & \textbf{0.487} & 0.487 & 0.487 & 0.487 & 0.487 & 0.487 & 0.487 & 0.487 & 0.487 & 0.487 & 0.487 & 0.487 & 0.487 & 0.487 & 0.487 & 0.487 \rule{0pt}{3.3ex}\cr
				& normMI  & 0.811      & 0.811 & \textbf{0.811} & \textbf{0.811} & \textbf{0.811} & 0.811 & 0.811 & \textbf{0.811} & 0.811 & 0.811 & 0.811 & 0.811 & 0.811 & 0.811 & 0.811 & 0.811 & 0.811 & 0.811 & 0.811 & 0.811 & 0.811 & 0.811 & 0.811 & 0.811 & 0.811 \cr
				& purity  & 0.663      & 0.663 & 0.663 & \textbf{0.663} & \textbf{0.663} & 0.663 & 0.663 & 0.663 & 0.663 & 0.663 & 0.663 & 0.663 & 0.663 & 0.663 & 0.663 & 0.663 & 0.663 & 0.663 & 0.663 & 0.663 & 0.663 & 0.663 & 0.663 & 0.663 & 0.663 \vspace{1.5ex} \cr \midrule
				\multirow{3}{*}{\begin{tabular}{@{}l@{}}Gonen's\cr LMKK \end{tabular}}  
				& adjRI   & 0.487      & \textbf{0.487} & \textbf{0.487} & \textbf{0.487} & \textbf{0.487} & 0.487 & 0.487 & \textbf{0.487} & \textbf{0.487} & \textbf{0.487} & 0.487 & 0.487 & 0.487 & 0.487 & 0.487 & 0.487 & 0.487 & 0.487 & 0.487 & 0.487 & 0.614 & 0.618 & 0.487 & 0.487 & 0.487 \rule{0pt}{3.3ex}\cr
				& normMI  & 0.811      & 0.811 & \textbf{0.811} & \textbf{0.811} & \textbf{0.811} & 0.811 & 0.811 & \textbf{0.811} & 0.811 & 0.811 & 0.811 & 0.811 & 0.811 & 0.811 & 0.811 & 0.812 & 0.812 & 0.811 & 0.811 & 0.811 & 0.964 & 0.965 & 0.811 & 0.811 & 0.811 \cr
				& purity  & 0.663      & 0.663 & 0.663 & \textbf{0.663} & \textbf{0.663} & 0.663 & 0.663 & 0.663 & 0.663 & 0.663 & 0.663 & 0.663 & 0.663 & 0.663 & 0.663 & 0.670 & 0.670 & 0.663 & 0.663 & 0.663 & 0.860 & 0.863 & 0.663 & 0.663 & 0.663 \vspace{1.5ex} \cr \midrule
				\multirow{3}{*}{Our Method} 
				& adjRI   & \textbf{0.632}      & \textbf{0.487} & \textbf{0.487} & \textbf{0.487} & \textbf{0.487} & \textbf{0.961} & \textbf{0.590} & \textbf{0.487} & \textbf{0.487} & \textbf{0.487} & \textbf{0.990} & \textbf{0.980} & \textbf{0.990} & \textbf{0.980} & \textbf{0.914} & \textbf{1.000} & \textbf{1.000} & \textbf{1.000} & \textbf{0.990} & \textbf{0.980} & \textbf{1.000} & \textbf{1.000} & \textbf{1.000} & \textbf{1.000} & \textbf{1.000} \rule{0pt}{3.3ex}\cr
				& normMI  & \textbf{1.009}      & \textbf{0.812} & \textbf{0.811} & \textbf{0.811} & \textbf{0.811} & \textbf{1.369} & \textbf{0.921} & \textbf{0.811} & \textbf{0.812} & \textbf{0.812} & \textbf{1.418} & \textbf{1.394} & \textbf{1.418} & \textbf{1.400} & \textbf{1.308} & \textbf{1.443} & \textbf{1.443} & \textbf{1.443} & \textbf{1.418} & \textbf{1.394} & \textbf{1.443} & \textbf{1.443} & \textbf{1.443} & \textbf{1.443} & \textbf{1.443} \cr
				& purity  & \textbf{0.840}      & \textbf{0.673} & \textbf{0.667} & \textbf{0.663} & \textbf{0.663} & \textbf{0.987} & \textbf{0.810} & \textbf{0.667} & \textbf{0.670} & \textbf{0.673} & \textbf{0.997} & \textbf{0.993} & \textbf{0.997} & \textbf{0.993} & \textbf{0.970} & \textbf{1.000} & \textbf{1.000} & \textbf{1.000} & \textbf{0.997} & \textbf{0.993} & \textbf{1.000} & \textbf{1.000} & \textbf{1.000} & \textbf{1.000} & \textbf{1.000} \vspace{1.5ex} \cr \bottomrule
			\end{tabular}%
			\begin{tablenotes}
				\footnotesize
				\item Clustering performance of the methods are evaluated by three widely-used metrics: Adjusted Rand Index (adjRI), Normalized Mutual Information (normMI), and purity. A higher value of the metrics indicates better clustering performance. Each column represents a simulated data set where the number indicates the number of the redundant variables ($\Nv_{redun}$) added to the first partial view (View 1) and correlation between each the redundant variables and the original variables. The bolded numbers are the maximum value for each the evaluation measure within a simulation data set.
			\end{tablenotes}
		\end{threeparttable}
	\end{adjustbox}	
\end{sidewaystable}
    \newpage
\begin{sidewaystable}
	\centering
	\caption[caption]{\textbf{Evaluation of the Clutering Methods on Scenario C-Redun}} \label{table:evalclusteringCredun}
	\begin{adjustbox}{max width=\textwidth}
		\def\arraystretch{0.95}
		\begin{threeparttable}
			\begin{tabular}{@{}llccccccccccccccccccccccccc@{}}
				\toprule
				\multicolumn{27}{l}{\large \textbf{Scenario C-2}} \cr \midrule
				& {\bf cor}
				& \multicolumn{5}{c}{{\bf 0.45}}   & \multicolumn{5}{c}{\bf 0.72}  & \multicolumn{5}{c}{\bf 0.90}    
				& \multicolumn{5}{c}{\bf 0.97}   & \multicolumn{5}{c}{\bf 1}     \cr \cmidrule(lr){3-7} \cmidrule(lr){8-12} \cmidrule(lr){13-17} \cmidrule(lr){18-22} \cmidrule(lr){23-27} 
				& $\Nv_{redun}$     
				& \textbf{1}    & \textbf{2}    & \textbf{3}    & \textbf{4}    & \textbf{5}    & \textbf{1}    & \textbf{2}    & \textbf{3}    & \textbf{4}    & \textbf{5}    & \textbf{1}    & \textbf{2}    & \textbf{3}    & \textbf{4}    & \textbf{5}    & \textbf{1}    & \textbf{2}    & \textbf{3}    & \textbf{4}    & \textbf{5}    & \textbf{1}    & \textbf{2}    & \textbf{3}    & \textbf{4}    & \textbf{5}    \cr \midrule 
                \multirow{3}{*}{\begin{tabular}{@{}l@{}}Single\cr Best\end{tabular}}
				& adjRI   & 0.487      & 0.487 & 0.487 & 0.487 & 0.487 & 0.487 & 0.487 & 0.487 & 0.487 & 0.487 & 0.487 & 0.487 & 0.487 & 0.487 & 0.487 & 0.487 & 0.487 & 0.487 & 0.487 & 0.487 & 0.487 & 0.487 & 0.487 & 0.487 & 0.487 \rule{0pt}{3.3ex}\cr
				& normMI  & 0.811      & 0.811 & 0.811 & 0.811 & 0.811 & 0.811 & 0.811 & 0.811 & 0.811 & 0.811 & 0.811 & 0.811 & 0.811 & 0.811 & 0.811 & 0.811 & 0.811 & 0.811 & 0.811 & 0.811 & 0.811 & 0.811 & 0.811 & 0.811 & 0.811 \cr
				& purity  & 0.663      & 0.663 & 0.663 & 0.663 & 0.663 & 0.663 & 0.663 & 0.663 & 0.663 & 0.663 & 0.663 & 0.663 & 0.663 & 0.663 & 0.663 & 0.663 & 0.663 & 0.663 & 0.663 & 0.663 & 0.663 & 0.663 & 0.663 & 0.663 & 0.663 \vspace{1.5ex} \cr \midrule
				\multirow{3}{*}{\begin{tabular}{@{}l@{}}Uniform\cr Weight\end{tabular}}
				& adjRI   & 0.488      & 0.487 & 0.487 & 0.487 & 0.487 & 0.914 & 0.487 & 0.487 & 0.487 & 0.487 & \textbf{0.990} & 0.970 & 0.494 & 0.487 & 0.487 & \textbf{1.000} & \textbf{1.000} & \textbf{1.000} & \textbf{0.990} & 0.970 & \textbf{1.000} & \textbf{1.000} & \textbf{1.000} & \textbf{1.000} & \textbf{1.000} \rule{0pt}{3.3ex}\cr
				& normMI  & 0.812      & 0.811 & 0.811 & 0.811 & 0.811 & 1.295 & 0.812 & 0.812 & 0.812 & 0.811 & \textbf{1.418} & 1.384 & 0.818 & 0.812 & 0.812 & \textbf{1.443} & \textbf{1.443} & \textbf{1.443} & \textbf{1.418} & 1.375 & \textbf{1.443} & \textbf{1.443} & \textbf{1.443} & \textbf{1.443} & \textbf{1.443} \cr
				& purity  & 0.673      & 0.667 & 0.670 & 0.667 & 0.667 & 0.970 & 0.677 & 0.667 & 0.667 & 0.667 & \textbf{0.997} & 0.990 & 0.703 & 0.673 & 0.677 & \textbf{1.000} & \textbf{1.000} & \textbf{1.000} & \textbf{0.997} & 0.990 & \textbf{1.000} & \textbf{1.000} & \textbf{1.000} & \textbf{1.000} & \textbf{1.000} \vspace{1.5ex} \cr \midrule
				\multirow{3}{*}{\begin{tabular}{@{}l@{}}MinMax \cr MinC\end{tabular}}
				& adjRI   & 0.488      & 0.487 & 0.487 & 0.487 & 0.487 & 0.914 & 0.487 & 0.487 & 0.487 & 0.487 & \textbf{0.990} & 0.970 & 0.494 & 0.487 & 0.487 & \textbf{1.000} & \textbf{1.000} & \textbf{1.000} & \textbf{0.990} & 0.970 & \textbf{1.000} & \textbf{1.000} & \textbf{1.000} & \textbf{1.000} & \textbf{1.000} \rule{0pt}{3.3ex}\cr
				& normMI  & 0.812      & 0.811 & 0.811 & 0.811 & 0.811 & 1.295 & 0.812 & 0.812 & 0.812 & 0.811 & \textbf{1.418} & 1.384 & 0.818 & 0.812 & 0.812 & \textbf{1.443} & \textbf{1.443} & \textbf{1.443} & \textbf{1.418} & 1.375 & \textbf{1.443} & \textbf{1.443} & \textbf{1.443} & \textbf{1.443} & \textbf{1.443} \cr
				& purity  & 0.673      & 0.667 & 0.670 & 0.667 & 0.667 & 0.970 & 0.677 & 0.667 & 0.667 & 0.667 & \textbf{0.997} & 0.990 & 0.703 & 0.673 & 0.677 & \textbf{1.000} & \textbf{1.000} & \textbf{1.000} & \textbf{0.997} & 0.990 & \textbf{1.000} & \textbf{1.000} & \textbf{1.000} & \textbf{1.000} & \textbf{1.000} \vspace{1.5ex} \cr \midrule
				\multirow{3}{*}{\begin{tabular}{@{}l@{}}Yu's\cr OKKC\end{tabular}} 
				& adjRI   & 0.487      & 0.487 & 0.487 & 0.487 & 0.487 & 0.487 & 0.487 & 0.487 & 0.487 & 0.487 & 0.487 & 0.487 & 0.487 & 0.487 & 0.487 & 0.869 & 0.487 & 0.487 & 0.487 & 0.487 & \textbf{1.000} & 0.519 & 0.487 & 0.487 & 0.487 \rule{0pt}{3.3ex}\cr
				& normMI  & 0.811      & 0.811 & 0.811 & 0.811 & 0.811 & 0.811 & 0.811 & 0.811 & 0.811 & 0.811 & 0.811 & 0.811 & 0.811 & 0.811 & 0.811 & 1.221 & 0.811 & 0.811 & 0.811 & 0.811 & \textbf{1.443} & 0.855 & 0.811 & 0.811 & 0.811 \cr
				& purity  & 0.663      & 0.663 & 0.663 & 0.663 & 0.663 & 0.663 & 0.663 & 0.663 & 0.663 & 0.663 & 0.663 & 0.663 & 0.663 & 0.663 & 0.663 & 0.953 & 0.663 & 0.663 & 0.663 & 0.663 & \textbf{1.000} & 0.737 & 0.663 & 0.663 & 0.663 \vspace{1.5ex} \cr \midrule
				\multirow{3}{*}{\begin{tabular}{@{}l@{}}Liu's\cr MKK-MIR\end{tabular}}
				& adjRI   & 0.487      & 0.487 & 0.487 & 0.487 & 0.487 & 0.487 & 0.487 & 0.487 & 0.487 & 0.487 & 0.487 & 0.487 & 0.487 & 0.487 & 0.487 & \textbf{1.000} & 0.487 & 0.487 & 0.487 & 0.487 & \textbf{1.000} & \textbf{1.000} & \textbf{1.000} & 0.487 & 0.487 \rule{0pt}{3.3ex}\cr
				& normMI  & 0.811      & 0.811 & 0.811 & 0.811 & 0.811 & 0.811 & 0.811 & 0.811 & 0.811 & 0.811 & 0.812 & 0.812 & 0.811 & 0.811 & 0.811 & \textbf{1.443} & 0.812 & 0.811 & 0.811 & 0.811 & \textbf{1.443} & \textbf{1.443} & \textbf{1.443} & 0.812 & 0.812 \cr
				& purity  & 0.663      & 0.663 & 0.663 & 0.663 & 0.663 & 0.667 & 0.663 & 0.663 & 0.663 & 0.663 & 0.670 & 0.670 & 0.667 & 0.663 & 0.663 & \textbf{1.000} & 0.670 & 0.663 & 0.663 & 0.663 & \textbf{1.000} & \textbf{1.000} & \textbf{1.000} & 0.670 & 0.670 \vspace{1.5ex} \cr \midrule
				\multirow{3}{*}{\begin{tabular}{@{}l@{}}Gonen's\cr MKK \end{tabular}}
				& adjRI   & 0.487      & 0.487 & 0.487 & 0.487 & 0.487 & 0.487 & 0.487 & 0.487 & 0.487 & 0.487 & 0.487 & 0.487 & 0.487 & 0.487 & 0.487 & 0.487 & 0.487 & 0.487 & 0.487 & 0.487 & 0.487 & 0.487 & 0.487 & 0.487 & 0.487 \rule{0pt}{3.3ex}\cr
				& normMI  & 0.811      & 0.811 & 0.811 & 0.811 & 0.811 & 0.811 & 0.811 & 0.811 & 0.811 & 0.811 & 0.811 & 0.811 & 0.811 & 0.811 & 0.811 & 0.811 & 0.811 & 0.811 & 0.811 & 0.811 & 0.811 & 0.811 & 0.811 & 0.811 & 0.811 \cr
				& purity  & 0.663      & 0.663 & 0.663 & 0.663 & 0.663 & 0.663 & 0.663 & 0.663 & 0.663 & 0.663 & 0.663 & 0.663 & 0.663 & 0.663 & 0.663 & 0.663 & 0.663 & 0.663 & 0.663 & 0.663 & 0.663 & 0.663 & 0.663 & 0.663 & 0.663 \vspace{1.5ex} \cr \midrule
				\multirow{3}{*}{\begin{tabular}{@{}l@{}}Gonen's\cr LMKK \end{tabular}}
				& adjRI   & 0.490      & 0.490 & 0.490 & 0.490 & 0.490 & 0.490 & 0.490 & 0.490 & 0.490 & 0.490 & 0.490 & 0.490 & 0.490 & 0.490 & 0.490 & 0.487 & 0.490 & 0.490 & 0.490 & 0.490 & 0.614 & 0.652 & 0.490 & 0.490 & 0.490 \rule{0pt}{3.3ex}\cr
				& normMI  & 0.814      & 0.814 & 0.814 & 0.814 & 0.814 & 0.814 & 0.814 & 0.814 & 0.814 & 0.814 & 0.814 & 0.814 & 0.814 & 0.814 & \textbf{0.814} & 0.811 & 0.814 & 0.814 & 0.814 & 0.814 & 0.964 & 0.984 & 0.814 & 0.814 & 0.814 \cr
				& purity  & 0.683      & \textbf{0.683} & \textbf{0.683} & \textbf{0.683} & \textbf{0.683} & 0.683 & 0.683 & \textbf{0.683} & \textbf{0.683} & \textbf{0.683} & 0.683 & 0.683 & 0.683 & 0.683 & 0.683 & 0.663 & 0.683 & 0.683 & 0.683 & 0.683 & 0.860 & 0.870 & 0.683 & 0.683 & 0.683 \vspace{1.5ex} \cr \midrule
				\multirow{3}{*}{Our Method}
				& adjRI   & \textbf{0.592}      & \textbf{0.498} & \textbf{0.498} & \textbf{0.498} & \textbf{0.498} & \textbf{0.941} & \textbf{0.656} & \textbf{0.498} & \textbf{0.498} & \textbf{0.498} & \textbf{0.990} & \textbf{0.990} & \textbf{0.980} & \textbf{0.827} & \textbf{0.499} & \textbf{1.000} & \textbf{1.000} & 0.990 & \textbf{0.990} & \textbf{0.980} & \textbf{1.000} & \textbf{1.000} & \textbf{1.000} & \textbf{1.000} & \textbf{1.000} \rule{0pt}{3.3ex}\cr
				& normMI  & \textbf{0.913}      & \textbf{0.837} & \textbf{0.837} & \textbf{0.837} & \textbf{0.837} & \textbf{1.343} & \textbf{0.988} & \textbf{0.838} & \textbf{0.837} & \textbf{0.837} & \textbf{1.418} & \textbf{1.418} & \textbf{1.400} & \textbf{1.168} & 0.801 & \textbf{1.443} & \textbf{1.443} & 1.418 & \textbf{1.418} & \textbf{1.400} & \textbf{1.443} & \textbf{1.443} & \textbf{1.443} & \textbf{1.443} & \textbf{1.443} \cr
				& purity  & \textbf{0.813}      & 0.680 & 0.680 & 0.680 & 0.680 & \textbf{0.980} & \textbf{0.853} & \textbf{0.683} & 0.673 & 0.680 & \textbf{0.997} & \textbf{0.997} & \textbf{0.993} & \textbf{0.937} & \textbf{0.743} & \textbf{1.000} & \textbf{1.000} & 0.997 & \textbf{0.997} & \textbf{0.993} & \textbf{1.000} & \textbf{1.000} & \textbf{1.000} & \textbf{1.000} & \textbf{1.000} \vspace{1.5ex} \cr \bottomrule
			\end{tabular}%
			\begin{tablenotes}
				\footnotesize
				\item Clustering performance of the methods are evaluated by three widely-used metrics: Adjusted Rand Index (adjRI), Normalized Mutual Information (normMI), and purity. A higher value of the metrics indicates better clustering performance. Each column represents a simulated data set where the number indicates the number of the redundant variables ($\Nv_{redun}$) added to the complete view and correlation between each the redundant variables and the original variables. The bolded numbers are the maximum value for each the evalution measure within a simulation data set.
			\end{tablenotes}
		\end{threeparttable}
	\end{adjustbox}	
\end{sidewaystable}

    \newpage
    \begin{table}[ht]
	\centering
	\caption[caption]{\textbf{List of BRCA/GBM related KEGG pathway}} \label{table:pathwaylist}
	\begin{adjustbox}{max width=\textwidth}
		\def\arraystretch{0.95}
		\begin{threeparttable}
			\begin{tabular}{@{}lcp{10cm}@{}}
				\toprule
				\small Group &  Size & Pathway\cr \midrule 
				\small Breast caner & 
						8	& 
						Estrogen signaling pathway;
						PI3K-Akt signaling pathway;
						Notch signaling pathway;
						Wnt signaling pathway;
						Homologous recombination;
						MAPK signaling pathway;
						p53 signaling pathway\cr
						Cell cycle
				\small Glioma & 
						7	&
						MAPK signaling pathway;
						p53 signaling pathway;
						Cell cycle;
						Cytokine-cytokine receptor interaction;
						ErbB signaling pathway;
						Calcium signaling pathway;
						mTOR signaling pathway\cr
				\small Pathways in cancer & 
						20 &
						Estrogen signaling pathway;
						PI3K-Akt signaling pathway;
						Notch signaling pathway;
						Wnt signaling pathway;
						MAPK signaling pathway;
						p53 signaling pathway;
						Cell cycle;
						Adherens junction;
						ECM-receptor interaction;
						Focal adhesion;
						cAMP signaling pathway;
						Jak-STAT signaling pathway;
						Hedgehog signaling pathway;
						HIF-1 signaling pathway;
						VEGF signaling pathway;
						Apoptosis;
						TGF-beta signaling pathway;
						Cytokine-cytokine receptor interaction;
						Calcium signaling pathway;
						mTOR signaling pathway
						\cr
				\small Central carbon metabolism in cancer & 
						10 &
						MAPK signaling pathway;
						PI3K-Akt signaling pathway;
						mTOR signaling pathway;
						HIF-1 signaling pathway;
						Alanine, aspartate and glutamate metabolism;
						Citrate cycle (TCA cycle);
						Fatty acid biosynthesis;
						Glycolysis / Gluconeogenesis;
						Glycine, serine and threonine metabolism;
						Oxidative phosphorylation
						\cr \bottomrule 
			\end{tabular}%
			\begin{tablenotes}
				\footnotesize
				\item The BRCA/GBM related biological pathways are provided by KEGG Pathway Database (\url{https://www.kegg.jp}), which defined independently from our data analysis. The list of BRCA related pathways is consist of pathways from Breast caner, Pathways in cancer, and Central carbon metabolism in caner. The list of GBM related pathways is consist of pathways from Glioma, Pathways in cancer, and Central carbon metabolism in cancer. 
			\end{tablenotes}
		\end{threeparttable}
	\end{adjustbox}	
\end{table}
    ~
    \newpage
    \begin{table}[ht]
	\centering
	\caption[caption]{\textbf{BRCA related KEGG pathways identified by our method}} \label{table:pathwayBRCA}
	\begin{adjustbox}{max width=\textwidth}
		\def\arraystretch{0.95}
		\begin{threeparttable}
			\begin{tabular}{@{}lcp{10cm}@{}}
				\toprule
				P-value & Cluster & Enriched Pathway \cr \midrule 
				0.080 & 3 & Cell cycle \cr 
				0.098 & 4 & Alanine, aspartate and glutamate metabolism \cr 
				0.037 & 5 & Focal adhesion \cr
				0.043 & 5 & Cytokine-cytokine receptor interaction \cr
				0.049 & 5 & ECM-receptor interaction \cr \bottomrule 
			\end{tabular}%
			\begin{tablenotes}
				\footnotesize
				\item BRCA related KEGG pathways identified by our method ($pvalue<0.1$) for each cluster and p-values from the pathway enrichment analysis are reported. P-values were adjusted to control the false discovery rate using the Benjamini-Hochberg procedure \cite{benjamini1995controlling}. 
			\end{tablenotes}
		\end{threeparttable}
	\end{adjustbox}	
\end{table}

    ~
    \newpage
    \begin{table}[ht]
	\centering
	\caption[caption]{\textbf{GBM related KEGG pathways identified by our method}} \label{table:pathwayGBM}
	\begin{adjustbox}{max width=\textwidth}
		\def\arraystretch{0.95}
		\begin{threeparttable}
			\begin{tabular}{@{}lcp{10cm}@{}}
				\toprule
				P-value & Cluster & Enriched Pathway \cr \midrule 
				0.000 & 1 & Focal adhesion \cr 
				0.001 & 1 & ECM-receptor interaction \cr 
				0.022 & 1 & MAPK signaling pathway \cr 
				0.056 & 1 & ErbB signaling pathway \cr 
				0.057 & 1 & Calcium signaling pathway \cr 
				0.095 & 1 & Adherens junction \cr 
				0.007 & 2 & Focal adhesion \cr 
				0.034 & 2 & ECM-receptor interaction \cr 
				0.090 & 3 & Calcium signaling pathway \cr 
				0.030 & 4 & Calcium signaling pathway \cr 
				0.072 & 4 & MAPK signaling pathway \cr \bottomrule 
			\end{tabular}%
			\begin{tablenotes}
				\footnotesize
				\item BRCA related KEGG pathways identified by our method ($pvalue<0.1$) for each cluster and p-values from the pathway enrichment analysis are reported. P-values were adjusted to control the false discovery rate using the Benjamini-Hochberg procedure \cite{benjamini1995controlling}. 
			\end{tablenotes}
		\end{threeparttable}
	\end{adjustbox}	
\end{table}

\end{document}